\documentclass[lettersize,journal]{IEEEtran}
\usepackage{amsmath,amsfonts}
\usepackage{amssymb}
\usepackage{tikz}
\usepackage{algorithmic}
\usepackage{array}
\usepackage{soul}
\usepackage{xcolor}
\usepackage{algorithmic}
\usepackage{algorithm}
\usepackage{cancel}
\usepackage{hyperref}
\usepackage{stmaryrd}
\usepackage{textcomp}
\usepackage{stfloats}
\usepackage{url}
\usepackage{verbatim}
\usepackage{graphicx}
\usepackage{subfigure}
\def\BibTeX{{\rm B\kern-.05em{\sc i\kern-.025em b}\kern-.08em
    T\kern-.1667em\lower.7ex\hbox{E}\kern-.125emX}}
\usepackage{balance}
\usepackage{amsthm}
\usepackage{tikz}
\usetikzlibrary{matrix}
\usepackage{tabularx}
\usepackage{academicons}
\usepackage{orcidlink}
\usepackage[backend=biber,style=ieee,maxnames=6]{biblatex}  
\newtheorem{theorem}{Theorem}[section]
\newtheorem{definition}{Definition}[section]
\newtheorem{lemma}[theorem]{Lemma}
\newtheorem{proposition}[theorem]{Proposition}
\newcommand{\norm}[1]{\left\lVert#1\right\rVert}
\newcommand{\orcid}[1]{\href{https://orcid.org/#1}{\textcolor[HTML]{A6CE39}{\aiOrcid}}}
\DeclareMathOperator{\Vol}{Vol}

\DeclareMathOperator*{\argmax}{\arg\!\max}
\allowdisplaybreaks
\usepackage{multirow}
\usepackage{colortbl}
\usepackage[most]{tcolorbox}
\AtEveryBibitem{%
  \iffieldundef{doi}
    {} 
    {\clearfield{url}\clearfield{urldate}} 
}

\usepackage{xcolor}
\usepackage{soul}      
\usepackage{hyperref}

\soulregister{\ref}{7}
\soulregister{\pageref}{7}
\soulregister{\eqref}{7}
\soulregister{\autoref}{7}
\soulregister{\cite}{7}
\soulregister{\url}{7}

\usepackage[final]{microtype} 
\begin{document}
\title{A Mathematical Certification for Positivity Conditions in Neural Networks with Applications to Partial Monotonicity and Trustworthy AI}
\author{
Alejandro Polo-Molina\,\orcidlink{0000-0001-7051-2288},
  David Alfaya\,\orcidlink{0000-0002-4247-1498},
  Jose Portela\,\orcidlink{0000-0002-7839-8982}
\thanks{This research was supported by funding from CDTI, with Grant Number MIG-20221006 associated with the ATMOSPHERE Project and grant PID2022-142024NB-I00 funded by MCIN/AEI/10.13039/501100011033. We would also like to thank the anonymous reviewers for their valuable comments, which have helped to improve the paper.}}

\markboth{}
{A mathematical certification for positivity conditions in Neural Networks with applications to partial monotonicity and Trustworthy AI}

\maketitle
\begingroup\renewcommand\thefootnote{}\footnotetext{%
{\footnotesize © 2025 IEEE. Personal use of this material is permitted. Permission from IEEE must be obtained for all other uses, in any current or future media, including reprinting/republishing this material for advertising or promotional purposes, creating new collective works, for resale or redistribution to servers or lists, or reuse of any copyrighted component of this work in other works.}%
}\addtocounter{footnote}{-1}\endgroup
\begin{abstract}
Artificial Neural Networks (ANNs) have become a powerful tool for modeling complex relationships in large-scale datasets. However, their black-box nature poses trustworthiness challenges. In certain situations, ensuring trust in predictions might require following specific partial monotonicity constraints. However, certifying if an already-trained ANN is partially monotonic is challenging. Therefore, ANNs are often disregarded in some critical applications, such as credit scoring, where partial monotonicity is required. To address this challenge, this paper presents a novel algorithm (LipVor) that certifies if a black-box model, such as an ANN, is positive based on a finite number of evaluations. 
Consequently, since partial monotonicity can be expressed as a positivity condition on partial derivatives, LipVor can certify whether an ANN is partially monotonic.
To do so, for every positively evaluated point, the Lipschitzianity of the black-box model is used to construct a specific neighborhood where the function remains positive. Next, based on the Voronoi diagram of the evaluated points, a sufficient condition is stated to certify if the function is positive in the domain. 
Unlike prior methods, our approach certifies partial monotonicity without constrained architectures or piece-wise linear activations.
Therefore, LipVor could open up the possibility of using unconstrained ANN in some critical fields. Moreover, some other properties of an ANN, such as convexity, can be posed as positivity conditions, and therefore, LipVor could also be applied.
\end{abstract}


\section{Introduction}\label{sec:1}

\IEEEPARstart{A}{rtificial} Neural networks (ANNs) have gained significant attention as a powerful tool for modeling complex non-linear relationships and state-of-the-art performance in many real-world applications \cite{Goodfellow2016DeepLearning}, \cite{Lecun2015DeepLearning}. 
Therefore, ANNs have been an extraordinarily active and promising research field in recent decades.
Its development is justified by the encouraging results obtained in many fields including speech recognition \cite{Hinton2012DeepGroups}, computer vision \cite{Voulodimos2018DeepReview}, financial applications \cite{Liu2017AApplications} and many others \cite{Sarvamangala2022ConvolutionalSurvey,Xu2020ApplicationIntelligence}. 

However, ANNs are considered black-box models as their analytical expression is hardly interpretable, and therefore, they can only be analyzed in terms of the inputs and outputs. Thus, ANNs can pose a significant challenge in fields where interpretability and transparency are often critical considerations  \cite{cohen2021}, \cite{Tjoa2021AXAI}. This need for explainability has caused the field of explainable artificial intelligence (XAI) to grow substantially in recent years \cite{Dosilovic2018ExplainableSurvey}. Consequently, there have been many approaches trying to explain how neural networks are computing their prediction \cite{Zhang2020AInterpretability,Pizarroso2022NeuralSens:Networks,Morala2023NN2Poly:Networks}. 

Nevertheless, explainability alone is insufficient for critical services such as medicine \cite{Tjoa2021AXAI} or credit scoring \cite{Bussmann2021ExplainableManagement}. As the number of training variables grows, ANNs risk capturing spurious or irrelevant patterns, and interpretability alone cannot prevent unfair predictions \cite{Rudin2019StopInstead}. Thus, training ANNs for such applications must ensure not only explainability but also robustness (reliable behavior under input perturbations) and fairness (unbiased predictions that do not disadvantage individuals). These principles underpin \textit{trustworthy AI}, where fairness and reliability are essential for deployment in high-stakes fields \cite{Kaur2021RequirementsReview,Thiebes2021TrustworthyIntelligence}.

One approach for ensuring that the model behaves appropriately is to incorporate prior knowledge from the human expert into the model. One example where leveraging prior expertise can enhance the model's fairness occurs when dealing with partial monotonicity constraints. By applying a partial monotonicity constraint, the model's output function is forced to be partially monotonic. Therefore, if an increasing (resp. decreasing) partial monotonicity constraint is imposed, then the model predictions should increase (resp. decrease) whenever a set of input values increases. 

Besides, in some cases, partial monotonicity is not just a matter of enhancing explainability and robustness but is often a requisite \cite{Xie2023TestingModelling}. Numerous studies have highlighted the critical role of monotonicity in fairness in some areas such as finance, health care, criminology, and education, where deviations from monotonicity may lead to misleading or biased human decisions \mbox{\cite{Liu2020CertifiedNetworks, Sivaraman2020Counterexample-GuidedNetworks, You2017DeepFunctions, Chen2023MonotonicityFinance}}.
For instance, in loan approval, it is coherent that an applicant with a better credit history has more possibilities of getting a loan approved. In cases where the credit history score (input) is not monotonic w.r.t. the loan approval probability (output), that would mean that clients with a better credit history are less prone to getting a loan $\textit{(ceteris paribus)}$. Therefore, the model would be generating unfair predictions. 
Thus, certifying partial monotonic predictions is crucial to guarantee fairness-aligned monotonicity constraints \cite{Wang2020DeontologicalConstraints}.


As a result, as highlighted in \cite{Runchi2023AnEffects, Dumitrescu2022MachineEffects}, logistic regression remains the standard approach in some fields such as credit industry due to its intrinsic interpretability and controlled behaviour, which aligns with the concerns of financial regulators \cite{EuropeanBankingAuthority2023MachineModels}. In contrast, ANNs are frequently dismissed due to their black-box nature, which could prevent us from knowing whether the model complies with the known mandatory monotonic relation \cite{cohen2021}. Consequently, training partial monotonic ANNs has been a relevant research field in recent years. To address this challenge, two main approaches have been developed \cite{Liu2020CertifiedNetworks}. First of all, constrained architectures could be considered so that monotonicity is assured \cite{Runje2023ConstrainedNetworks,Daniels2010MonotoneNetworks,You2017DeepFunctions}. Although any of these methods guarantee partial monotonicity, their architecture can be very restrictive or complex and difficult to implement \cite{Liu2020CertifiedNetworks, Mikulincer2024SizeApproximation}.


On the other hand, monotonicity can be enforced by adding a regularization term during the learning process.  \cite{Sivaraman2020Counterexample-GuidedNetworks} proposes a method to find counterexamples where the monotonicity is unmet. Besides, \cite{Gupta2019HowFlexibility} opted for sampling instances from the input data to compute a partial monotonic regularization term. Based on this idea, \cite{Monteiro2022MonotonicityClassification} computes a penalization term at random points sampled inside the convex hull defined by the input data. Although more flexible than constrained architectures, these methods cannot guarantee monotonicity across the entire input space, making them unsuitable in domains where regulators require monotonic models to ensure fair predictions \cite{EuropeanBankingAuthority2023MachineModels}.


Despite the numerous studies regarding ANNs' training towards partially monotonic solutions, fewer efforts have been made to certify the partial monotonicity of an already trained ANN \cite{Liu2020CertifiedNetworks,Sharma2020TestingModels}. For instance, \cite{Liu2020CertifiedNetworks} proposes an optimization-based technique to certify monotonicity for ANNs trained with piece-wise linear activation functions such as ReLU or Leaky ReLU. However, their method involves solving a mixed integer linear programming (MILP) problem, which becomes computationally expensive as the number of neurons increases, exhibiting exponential growth in complexity. Additionally, since the MILP method involves solving an NP problem, it may not always yield conclusive results: if the solver fails to find a strictly positive lower bound for the partial derivatives, it cannot confirm partial monotonicity, leaving the verification inconclusive.

On the other hand, \cite{Sharma2020TestingModels} proposes using a decision tree trained to approximate a black-box model and using an Satisfiability Modulo Theories  (SMT) solver to find possible counter-examples for partial monotonicity. However, the proposed algorithm does not guarantee finding counter-examples. Furthermore, if the ANN is truly partially monotonic, the algorithm is unable to identify a counter-example and conclusively determine whether the ANN is partially monotonic. Therefore, this method cannot be used to obtain a complete mathematical proof of the partial monotonicity of the model. Consequently, to the best of our knowledge, this is the first study presenting an external certification algorithm to certify whether a trained unconstrained ANN, or any black-box model, is partially monotonic without considering constrained architectures or piece-wise linear activation functions.

Although no external certification algorithm exists in the literature, one that determines whether a neural network is partially monotonic without constrained architectures would be highly valuable. For example, Article 179(1)(a) of the EU Capital Requirements Regulation (575/2013) (CRR) \cite{CapitalAuthority} requires internal rating-based models (IRB) to produce plausible and intuitive estimates. Accordingly, the European Banking Authority (EBA) states that, for each IRB model, the economic relationship between each risk driver and the output variable should be evaluated to verify plausibility and intuitiveness \cite{EuropeanBankingAuthority2023MachineModels}. In loan approval, for instance, it is neither intuitive nor plausible for credit history and loan approval probability to be non-monotonic. Thus, regulatory agencies enforce monotonicity constraints to ensure fairness and interpretability. Hence, an external certification algorithm verifying partial monotonicity would provide independent assurance.

Following this premise, this paper presents a novel approach to certify if an already-trained unconstrained ANN is partially monotonic. To accomplish this, we propose a novel methodology to solve a broader problem: mathematically certifying that a black-box model remains positive over its entire domain. Therefore, as increasing (or decreasing) partial monotonicity can be assessed by checking the positive (negative) sign of the partial derivatives, certifying partial monotonicity is equivalent to checking the positivity of the partial derivatives. For this purpose, a novel algorithm is presented capable of determining whether a black-box model is positive in its domain based on a finite set of evaluations.

To implement this approach, the algorithm leverages the model's Lipschitz continuity to establish specific neighborhoods around each positively evaluated point, ensuring the function remains positive within these neighborhoods. By utilizing Voronoi diagrams generated from the evaluated points and their corresponding neighborhoods, a sufficient condition is derived to ensure the function's positivity throughout the entire domain. Thus, this paper presents a novel approach that combines the analytical properties of the black-box model with the geometry of the input space to certify partial monotonicity. Moreover, based on the aforementioned algorithm, this study introduces a novel methodology to train unconstrained ANNs that can be later certified as partial monotonic. 

Although the Lipschitzianity has already been studied as a natural way to analyze the robustness \cite{Pauli2022TrainingBounds} and fairness \cite{Dwork2012FairnessAwareness} of an ANN, this paper utilizes the Lipschitz continuity in a novel approach to extend point-wise positivity, i.e., positivity at a point, to positivity at a neighborhood of the point. The exact computation of the Lipschitz constant of an ANN, even for simple network architectures, is NP-hard \cite{Virmaux2018LipschitzEstimation}. Nevertheless, some studies present methodologies to generate estimates of the Lipschitz constant \cite{Virmaux2018LipschitzEstimation},\cite{Shi2022EfficientlyPropagation}. However, this paper introduces, for the first time\footnote{
Although in \cite{Virmaux2018LipschitzEstimation} a general method for estimating the Lipschitz constant of a function computable in K operations is given, the specific computable expression of the partial derivatives of an ANN, which is non-trivial, is not provided.}, a specific estimate of the Lipschitz constant of the partial derivative of a neural network. 

On the other hand, the relationship between the Lipschitzianity of an ANN and partial monotonicity has also been explored. \cite{Kitouni2021RobustNetworks} proposes to normalize the weights of the ANN to achieve a predefined Lipschitz constant. Then, by adding a linear term multiplied by the imposed Lipschitz constant to the trained ANN, a monotonic residual connection can be used to make the model monotonic. However, it requires knowing the Lipschitz constant of the estimated function in advance. Besides, achieving the predefined Lipschitz constant imposes a huge weight normalization specifically for deep neural networks. Moreover, this method cannot be used to certify the partial monotonicity of an already-trained ANN.

Finally, partial monotonicity is not the only property that can be framed as a positivity constraint in ANNs. For example, certifying convexity reduces to checking positivity of second derivatives, making the proposed methodology applicable. Convexity has so far been studied only under restrictive architectural assumptions, and no general sufficient conditions or methods exist for certifying it in already trained ANNs \cite{Sivaprasad2021TheNetworks}.


The paper is structured as follows: Section \ref{sec:lipvor} introduces the LipVor Algorithm for positivity certification of a black-box model. Section \ref{sec:mono_nn} presents ANN partial monotonicity certification and the proposed Lipschitz upper bound for the partial derivatives. Section \ref{sec:mono_training} details the training methodology for unconstrained certified partial monotonic ANNs. Section \ref{sec:implementation_details} describes computational aspects, including libraries and complexity analysis. Section \ref{sec:case_studies} illustrates the approach through case studies. Finally, Section \ref{sec:conclusion} summarizes the contributions and results. 
The algorithm and experiments are available at \href{https://github.com/alejandropolo/LipVor}{https://github.com/alejandropolo/LipVor}.

\section{Positivity Certification: The LipVor Algorithm}\label{sec:lipvor}

As mentioned in Section \ref{sec:1}, many properties of a function can be stated in terms of a positivity condition. For instance, for a continuously differentiable function $f: \Omega \subseteq \mathbb{R}^n \rightarrow \mathbb{R}$, being increasing partially monotonic w.r.t. the $r^{th}$ input is equivalent to having positive $r^{th}$ partial derivative $\left(\frac{\partial f}{\partial x_r}>0 \right)$. Therefore, certifying partial monotonicity can be posed as a positivity certification problem of the partial derivatives. 

However, for a black-box function that can only be point-wise evaluated, determining positivity in its entire domain is challenging. This is particularly relevant for ANNs, where, despite knowing their analytical expression, verifying properties such as the positivity of partial derivatives across the entire domain is highly complex. As a result, ANNs are often treated as black-box models, where only the input-output relationship is accessible. Therefore, any analysis, such as verifying properties like partial monotonicity, can only be done through pointwise evaluations rather than direct analytical methods.

To address this challenge, this section presents an algorithm to certify the positivity of a black-box based on the evaluation of a finite set of points. Hence, for a black-box model $f$ and a finite set of positively evaluated points $\mathcal{P}$, we will utilize the Lipschitzianity of $f$ to state a specific neighborhood of each point where the function is also positive. 
Consequently, given a finite set of positively evaluated points, a sufficient condition will be given to determine whether the function is certified positive in the whole input domain $\Omega$.

\subsection{Local positivity Certification}

First of all, let us present the methodology to extend point-wise positivity to neighborhoods of the points where the function is also positive. Therefore, given a domain $\Omega \subseteq \mathbb{R}^n$, a point $x \in \Omega$ and a Lipschitz continuous function $f:\Omega \rightarrow \mathbb{R}$ $(f \in C^{0,1}(\Omega))$ such that $f(x)>0$, a specific neighborhood of $x$ will be stated where the function is certified positive.

By continuity of $f$ in $\Omega$, it can be proven that if $f(x)>0$, then there exists a neighborhood of $x$ where $f$ remains positive. However, just relying on the continuity of $f$ does not allow us to pinpoint a specific neighborhood. On the other hand, leveraging the Lipschitz continuity of $f$ enables us to precisely determine a specific ball centered at $x$, where positivity can be certified. Hence, point-wise positivity may be extended to neighborhoods where the function is also positive.

Therefore, let us start by presenting the Lipschitzianity of a function. Intuitively, for an L-Lipschitz function the output variation is bounded by a constant L, called the Lipschitz constant, and the variation of the input. 
Specifically, for real-valued functions under the Euclidean norm, we have the following definition  (cf. Def. 5.5.3 \cite{Sohrab2014BasicEdition}).

\begin{definition}\label{def:Lip}
    A function $f:\Omega \subseteq \mathbb{R}^n \rightarrow \mathbb{R}$ is said to be 
    L-Lipschitz in the $L^2$ norm (or simply L-Lipschitz or  Lipschitz continuous) if there exists a constant $L \geq 0$ such that:
    \begin{equation}\label{eq:Lip}
        |f(x)-f(y)| \leq L \|x-y\|, \quad \forall x,y \in \mathbb{R}^n,
    \end{equation}

    where $\|x-y\|$ refers to the $L^2$ norm \footnote{As we will later explain, in this paper, we will only use the $L^2$ norm unless stated otherwise. Therefore, although the notation $\| \cdot \|_2$ will not be explicitly used, every norm considered will be the $L^2$ norm.}.  Any such $L$ verifying Eq. \eqref{eq:Lip} is called a Lipschitz constant of the function and the smallest constant is the (best) Lipschitz constant. 
\end{definition}
Although Lipschitzianity of $f$ might seem at first as a strong condition to be assumed, it is worth noting that if a function $f$ is continuously differentiable in a compact domain $(f \in C^1(\Omega))$, then $f$ is Lipschitz continuous. Moreover, in a compact convex set $\Omega_C$, the Lipschitz constant of a function  $f\in C^1(\Omega_C)$ is the maximum norm of its gradient (Theorem 3.1.6
\cite[Rademacher]{Federer1996GeometricTheory}).

As mentioned before, using the Lipschitzianity of a function it is possible to determine a neighborhood in which the positivity is certified. Therefore, for each positively evaluated point $x \in \Omega$, we can determine an open ball $B(x,\delta) = \{p \in \Omega \mid \norm{x-p} < \delta\}$, centered at the point $x$ with a specific radius $\delta$, where the constraint is also fulfilled. 

Starting from a L-Lipschitz function $f \in C^{0,1}(\Omega)$  and $x_0 \in \Omega$ such that $f(x_0)> 0$, by definition of L-Lipscthitz
\begin{equation}\label{eq:lip_mono}
    |f(x_0)-f(x)| \leq L \|x_0-x\|.
\end{equation}
Consequently, taking $\delta_0 = \frac{f(x_0)}{L}$ and $x \in B(x_0,\delta_0)$, Eq. \eqref{eq:lip_mono} states that
\(
   |f(x_0)-f(x)| \leq L \|x_0-x\| < L \delta_0 = \cancel{L} \frac{f(x_0)}{\cancel{L}} = f(x_0). 
\)
Therefore, checking both sides of the inequality
\(
    |f(x_0)-f(x)|  < f(x_0) \implies f(x) >0,
\)
because if $f(x)<0$ then $|f(x_0)-f(x)| = f(x_0)-f(x) < f(x_0) \iff -f(x) < 0 \iff f(x)>0$ which would be a contradiction.

Hence, leveraging the Lipschitzianity of $f$ allows us to construct specific neighborhoods of $x$ where the positivity is verified whenever positivity is satisfied at $x$.
\begin{proposition}\label{prop:rad_mono}
    Let $f: \Omega \subseteq \mathbb{R}^n \rightarrow \mathbb{R}$ with $f \in C^0(\Omega)$ and $x_0 \in \Omega $. If  $f$ is L-Lipschitz and $f(x_0)>0$, then there exists a radius $\delta_0 = \frac{f(x_0)}{L}$ such that $f(x)> 0$, $\forall x \in B(x_0,\delta_0)$. 
\end{proposition}

\subsection{Global Positivity Certification in a Compact Domain}

As stated in Proposition \ref{prop:rad_mono}, for each point $x_0 \in \Omega$ verifying the positivity condition, there is a ball $ B(x_0,\delta_0)$ where the condition is also fulfilled. Consequently, to check whether a function is positive in a compact domain $\Omega$, the problem reduces to verifying if, for a given set of points $\mathcal{P}= \{p_1,p_2,\dots,p_k\}$ and the obtained radii of certified positivity $\mathcal{R}=\{\delta_1,\delta_2,\dots,\delta_n\}$, the union of the respective balls centered at each point covers $\Omega$. If the union of balls covers $\Omega$, that would mean that, for every point $x \in \Omega$, there is a sufficiently close $p_i$ such that the positivity certification at $p_i$ extends to $x$.

Checking if this condition is fulfilled in N-dimensional spaces is not trivial. However, based on Voronoi diagrams \cite{Okabe2000SpatialKendall}, a sufficient condition can be stated to determine if a set of balls covers $\Omega$. 

A Voronoi diagram divides the input space into cells, with each cell associated with a specific point from a given set $\mathcal{P}$. In each cell, the point that is closest to any arbitrary point within that region is the one that defines the boundary of that cell. Formally, Voronoi diagrams can be defined as follows.

\begin{definition}

Let $\mathcal{P} = \{p_1, p_2, \dots, p_k\}$ be a set of $k$ distinct points (sites) in the Euclidean space $\mathbb{R}^n$. The Voronoi cell $R_i$ associated with a point $p_i$ is defined as the set of all points $x$ in $\mathbb{R}^n$ whose distance to $p_i$ is less than or equal to its distance to any other point in $\mathcal{P}$:
\[ R_i = \{x \in \mathbb{R}^n \mid \|x - p_i\| \leq \|x - p_j\|\; \forall j \neq i, 1 \leq j \leq k\}. \]

\end{definition}
This definition implies that $R_i$ contains all points closer to $p_i$ than any other point of $\mathcal{P}$. Hence, the Voronoi cell $R_i$ forms a convex polytope and is bounded by hyperplanes, where each hyperplane represents the locus of points equidistant between $p_i$ and one of its neighbouring sites. The set of all Voronoi cells $(R_i)_{1\leq i\leq k}$ constitutes the Voronoi diagram $\text{V}(\mathcal{P} )$ of the set of points $\mathcal{P}$.  For instance, in a 2-dimensional space, each Voronoi cell is represented as a convex polygon (Fig \ref{fig:vor_diag}).

Therefore, the Voronoi diagram $\text{V}(\mathcal{P} )$ presents a partition of the compact space $\Omega$ in Voronoi cells $(R_i)_{1\leq i\leq k}$  generated by each of the initials points in $\mathcal{P}$. Hence, if a ball of radius $\delta_j$ is placed centered at each $p_j \in \mathcal{P}$ such that $\delta_j$ is greater than the distance from $p_j$ to its furthest point of the Voronoi cell $R_j$, then the ball $B(p_j,\delta_j)$ intuitively covers $R_j$. Consequently, $\Omega$ would be covered by the union of balls $\bigcup_{1\leq i \leq k} B(p_i, \delta_i)$ as each ball covers its corresponding Voronoi cell. This idea is mathematically stated and proved in Lemma \ref{lem:covering}.

Therefore, consider a given L-Lipschitz function $f$ and a set of points $\mathcal{P} = \{p_1, p_2, \dots, p_k\}$ with  $\delta_j = \frac{f(p_j)}{L}$ the radius of extended positivity given by Proposition \ref{prop:rad_mono}. Then if $\delta_j$
is greater than the maximum distance from each $p_j$ to the furthest point of $R_j$, for all $p_j \in \mathcal{P}$,  each $B(p_j,\delta_j)$ covers its corresponding Voronoi cell $R_j$. Consequently, each Voronoi cell is certified positive and hence $f$ is certified positive in $\Omega$. 

This intuitive idea is mathematically proven in Theorem \ref{teo:glob_condition}, which states a sufficient condition for certified positivity. The complete proof of Theorem \ref{teo:glob_condition}  can be found in appendix \ref{sec:ap_glob_condition}. 

\begin{theorem}\label{teo:glob_condition}
    Let $f$ be a L-Lipschitz function and $\mathcal{P} = \{p_1, p_2, \dots, p_k\}$ a set of points in a compact domain $\Omega$. Set $\delta_j =  \frac{f(p_j)}{L}$, $\forall j \in I = \{1,2,\dots,k\}$ the radius of extended positivity and $V(\mathcal{P})=(R_j)_{j\in I}$ the Voronoi diagram of $\mathcal{P}$, then the function is positive in $\Omega$ if 
    \begin{equation}\label{eq:glob_condition}
    \max_{x \in R_j} d(x,p_j) < \delta_j, \forall j \in I .
    \end{equation}
\end{theorem}
Note that the radius of certified positivity depends on the value of the evaluation of $f$ at the points of $\mathcal{P}$. In particular, whenever $f(p_j)=0$, the radius $\delta_j=0$. However, when considering the certification of positive functions in a compact domain $\Omega$, as stated in Lemma \ref{lem:eps_strict_pos}, there exists an $\varepsilon_{f,\Omega} >0$ such that
\begin{equation}\label{eq:eps_mono}
    f(x) \geq \varepsilon_{f,\Omega}, \forall \; x \in \Omega,
\end{equation}
so the radius of extended positivity will always be greater than 0. In such cases where there exists an $\varepsilon >0$ verifying Eq. \eqref{eq:eps_mono}, $f$ is said to be $\varepsilon$-positive. Therefore, by Lemma \ref{lem:eps_strict_pos}, every positive function in a compact domain $\Omega$ is  $\varepsilon_{f,\Omega}$-positive.

\begin{figure}[!htbp]
\centering
\includegraphics[width=2.5in]{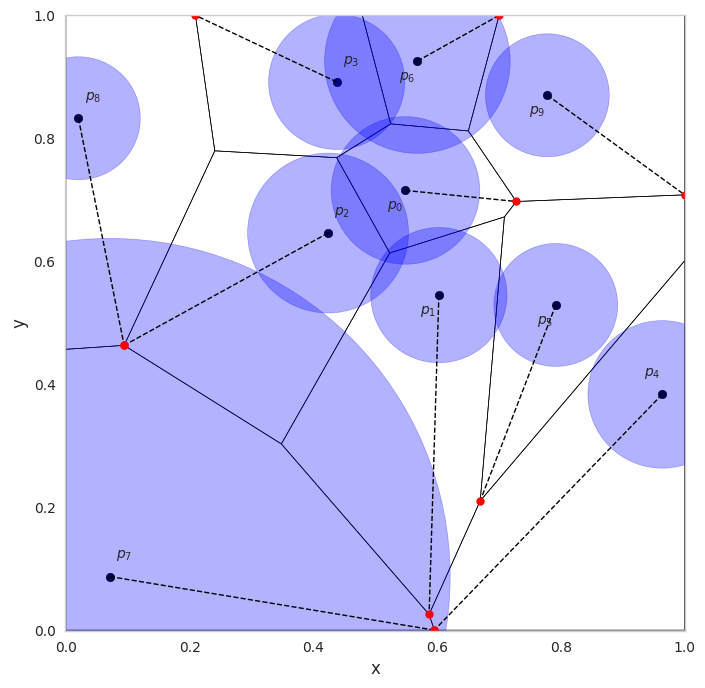}
\caption{Voronoi Diagram $V(\mathcal{P})$ for a set of 10 randomly allocated  points $ \mathcal{P} = \{p_0,p_1,\dots,p_{9}\}$ in a 2D space. Each red point represents the furthest vertex to each of the points in $\mathcal{P}$ and each circle is the ball of certified positivity given by Proposition \ref{prop:rad_mono}.}
\label{fig:vor_diag}
\end{figure}

\subsection{The LipVor Algorithm}\label{subsec:lipvor_algo}
After establishing a sufficient condition for the positivity certification of a function in Theorem \mbox{\ref{teo:glob_condition}}, we present a novel algorithm, LipVor, which determines in a finite number of steps whether a function is certified positive or is not $\varepsilon$-positive. As mentioned before, by Lemma \ref{lem:eps_strict_pos}, for a continuous function in a compact domain, $\varepsilon$-positivity is equivalent to positivity for a specific $\varepsilon$. Consequently, for a sufficiently small $\varepsilon$, concluding that the function is not $\varepsilon$-positive is equivalent to not being positive.

Before exploring the detailed formulation, we first introduce the core intuition behind the LipVor algorithm, which will later be expanded. First, leveraging the Lipschitzianity of the function $f$, it can be assessed, for each positively evaluated point, a ball centered at the point within which the function remains positive. Therefore, given a set of such balls corresponding to the verified positive points, the algorithm checks whether their union covers the entire domain $\Omega$. To achieve this, the space is partitioned into Voronoi cells using the Voronoi diagram generated from the set of points. If each ball covers the furthest vertex of its corresponding Voronoi cell, the entire domain is covered, and the positivity of the function across $\Omega$ is certified. If the domain is not fully covered and no counterexample to positivity is found, the algorithm expands the Voronoi diagram by appending one of the furthest vertices, iterating this process until full coverage is achieved or a counterexample is identified.

Given a function $f: \Omega \subset\mathbb{R}^n \rightarrow \mathbb{R}$,  defined in a compact domain $\Omega$, and any set of points $\mathcal{P} = \{p_1, p_2, \dots, p_k\}\subset \Omega$ if Eq. \eqref{eq:glob_condition} from Theorem \ref{teo:glob_condition} is not fulfilled, then the initial set of $k$ points $\mathcal{P}$ is not sufficient to guarantee positivity of $f$ in $\Omega$.
For example, Figure \ref{fig:vor_diag} represents a Voronoi diagram in 2D where Eq. \eqref{eq:glob_condition} is not fulfilled. Therefore, to try to certify partial monotonicity, LipVor presents a method of selecting points from $\Omega$ that are added to the initial sample $\mathcal{P}$ until Eq. \eqref{eq:glob_condition} is verified or a counter-example is found.

The idea regarding the LipVor Algorithm \ref{alg:alg1} is the following. Consider $\mathcal{P} = \{p_1, p_2, \dots, p_k\}$ a set of points in $\Omega$. The first step of the algorithm is to check that the value of the function at $\mathcal{P}$ is greater or equal than a certain $\varepsilon$. In that case, that would mean that the function fulfills the $\varepsilon$-positivity constraint in $\mathcal{P}$. If for any $p_j \in \mathcal{P},\; f (p_j) < \varepsilon$, then the algorithm would have already found a counter-example. Otherwise, $\varepsilon$-positivity is verified at $\mathcal{P}$ and the next step is to check if the local positivity condition at each point extends to a global positivity condition by Theorem \ref{teo:glob_condition}.

In case that Eq. \eqref{eq:glob_condition} is not satisfied, the LipVor Algorithm iteratively selects a point $p_{k+1} \in  \Omega$ to try to fulfill the aforementioned condition. The heuristic of selection of the point $p_{k+1}$ is the following. First of all, for each $p_j \in \mathcal{P}$, the furthest point $v_j$ in its Voronoi cell $R_j$ is computed. As each Voronoi cell is a convex polytope, the distance function attains its maximum in one of the vertex of the polytope. Therefore, to obtain the furthest point $v_j$ to the point $p_j$ generating the Voronoi cell $R_j$, the distance to each vertex of $R_j$ is computed. 

After computing the list of furthest vertices $\mathcal{V}=\{v_1, v_2, \dots, v_k\} $ for each Voronoi cell, the next point to be added to the Voronoi diagram is selected. Each of the furthest vertices $v_j$ is related to at least one parent point $p_j$ by the relation $v_j = \argmax_{x \in R_j} d(x,p_j)$. Recall that if $\delta_j$ is greater than the distance of each parent point $p_j$ to its furthest vertex $v_j$, then the Voronoi cell $R_j$ is covered by the ball with center $p_j$ and radius $\delta_j$. Therefore, starting from $\mathcal{V}$, those vertices already covered by the open ball centered at its parent point are discarded as the corresponding Voronoi cell is already certified positive. Therefore, the list of vertices is reduced to $\mathcal{V}_p = $$\{v_j \mid \delta_j \leq d(p_j,v_j) \}$. For instance, in Figure \ref{fig:vor_diag}, the only regions where $\delta_j$ is greater than the distance from the furthest vertex to the parent point $p_j$ are $R_6$ and $R_7$.

Considering the reduced list of vertices $\mathcal{V}_p$, the next point to add to the Voronoi diagram is selected based on the value of $f$ at the parent point $p_j$ and the number of adjacent balls ($n_{v_j}$) covering the vertex $v_j$. The idea of this procedure is to try to fill $\Omega$ with the least number of iterations possible as larger balls should cover the space faster. As the radius of certified positivity $\delta_j = \frac{f(p_j)}{L}$ is proportional to the value of $f$ at $p_j$, by continuity, the expected greatest value of $f$ in $\mathcal{V}_p$ is the one corresponding to the greatest value of $f$ in $\mathcal{P}$. Moreover, if the vertex is covered by some of the adjacent balls, the expected non-covered area of the space that could potentially be filled with the added point could be lesser. The way to measure the number of adjacent balls  covering the vertex is
\begin{equation}\label{eq:covered_vertices}
    n_{v_j} = |\{p_l \mid d(p_l,v_j) \leq \delta_l, \forall \, 0\leq l \leq k, l\not = j \}| . 
\end{equation}
Besides, with probability $0\leq p \leq 1$ the selected vertex $v_j$ is the one which corresponding parent $p_j$ has the smallest radius and again minimum $n_{v_j}$. The idea is to establish a trade-off between exploration and exploitation such as in Reinforcement Learning \cite{Sutton2018ReinforcementEd.}. Whenever the furthest vertex with the greatest value of the parent $f(p_j)$ is selected (exploitation) the best option for the next added point is chosen based on the current knowledge. On the other hand, selecting the vertex corresponding to the smallest parent's value (exploration) corresponds to trying new options that may lead to finding counter-examples.  Algorithm \ref{alg:alg1} depicts the procedure of LipVor.

\begin{algorithm}[!htbp]
\caption{LipVor}\label{alg:alg1}
\begin{algorithmic}

\STATE {\textbf{Input:} Function $f$, Lipschitz constant $L$, positivity constant $\varepsilon>0$, number of maximum iterations $N$ and a set of points $\mathcal{P} = \{p_1, p_2, \dots, p_k\}$.}
\STATE {\textbf{Output: } Bool variable ($\textit{isPositive}$) reflecting if $f$ is certified positive or not, counter-example (\textit{cExample}), if any, and next vertex to be added (\textit{nVertex}).}
\STATE {\textbf{Initialize} $\textit{isPositive} = $ True, $\textit{cExample} =$ None, $\textit{nVertex} =$ None, $\textit{counter} = 0$, $n_{\min}=\infty$ and $\delta_{\max}= -\infty$}.
\STATE {\textbf{Compute} the Voronoi diagram $\mathcal{V}(\mathcal{P})=\left(R_j\right)_{1\leq j\leq k}$;}
\STATE {\textbf{For} $n=1 \rightarrow N$ \textbf{do}:}
\STATE \hspace{0.5cm}{\textbf{For} $j=1 \rightarrow \text{length}(\mathcal{P})$ \textbf{do}:}
\STATE \hspace{1.0cm}{\textbf{If} $f(p_j)<\varepsilon$:}
\STATE \hspace{1.5cm} {\textit{isPositive} = False;}
\STATE \hspace{1.5cm} {\textit{cExample} = $p_j$;}
\STATE \hspace{1.5cm} {\textbf{Return} \textit{isPositive}, \textit{cExample} and \textit{nVertex} }
\STATE \hspace{1.0cm}{\textbf{Else}:}
\STATE \hspace{1.5cm} {Compute the furthest vertex:}
\STATE \hspace{2.0cm} {$v_j = \argmax_{x \in R_j} d(x,p_j)$;}
\STATE \hspace{1.5cm} {Compute the distance to the furthest vertex:}
\STATE \hspace{2.0cm} {$d_j = d(p_j,v_j)$;}
\STATE \hspace{1.5cm} {\textbf{If} $\delta_j = \frac{f(p_j)}{L}  \leq d_j $: }
\STATE \hspace{2.0cm} {Using Eq. \eqref{eq:covered_vertices} find $n_j$;}  
\STATE \hspace{2.0cm} {\textbf{If} $\delta_j \geq \delta_{\max} $ and $n_j \leq n_{\min}$:}
\STATE \hspace{2.5cm}{$n_{\min} = n_j  $;}
\STATE \hspace{2.5cm}{$\delta_{\max} = \delta_j $;}
\STATE \hspace{2.5cm}{$\textit{nVertex} = p_j  $;}
\STATE \hspace{2.5cm}{\textit{isPositive} = False;}
\STATE \hspace{1.5cm} {\textbf{Else} : }
\STATE \hspace{2.0cm} {\textit{counter +=1}}
\STATE \hspace{0.5cm} {\textbf{If} $\textit{counter} = \text{length}(\mathcal{P})$: }
\STATE \hspace{1.0cm} {\textbf{Return} \textit{isPositive}, \textit{cExample} and \textit{nVertex}} 
\STATE \hspace{0.5cm}{\textbf{End for}}
\STATE \hspace{0.5cm}{Add point $\textit{nVertex}$ to $\mathcal{P}$;}
\STATE \hspace{0.5cm}{Compute new Voronoi diagram $\mathcal{V}(\mathcal{P})$;}
\STATE \hspace{0.0cm}{\textbf{End for}}
\STATE \hspace{0.0cm}{\textbf{Return} \textit{isPositive}, \textit{cExample} and \textit{nVertex}. }
\end{algorithmic}
\label{alg1}
\end{algorithm}

In practice, the \hyperref[alg:alg1]{LipVor Algorithm} is slightly modified to find not just one counter-example but a list of them. The idea is to expand not just the probably greatest and least covered vertices but also those whose parent's evaluation is less than $\varepsilon$ and the value is the smallest (greatest in absolute value). Consequently, even if for any $p_j \in \mathcal{P}, f (p_j) < \varepsilon$, the point $p_j$ is also considered as a valid parent to find its furthest vertex. Therefore, the LipVor Algorithm would also expand the subdomain where the positivity is not verified. Hence, when using the LipVor Algorithm to certify if a function fulfills a positivity condition, the algorithm returns whether the function is positive and the subdomains where the condition is certified met and where it is not.

As stated in the pseudocode, the LipVor Algorithm ends in a finite number of $N$ steps. The following Theorem \ref{teo:convergence} shows that a value of $N$ can always be chosen depending only on $\Omega$, $\varepsilon$ and $L$ so that \hyperref[alg:alg1]{LipVor} always reaches a definitive conclusion within the given number of iterations. A complete proof of Theorem \ref{teo:convergence} is given in appendix \ref{sec:ap_bound_iterations}. 

\begin{theorem}\label{teo:convergence}
    Under the aforementioned conditions, the LipVor Algorithm concludes in a finite number of steps. Moreover, the maximum number of steps $N$ that the LipVor Algorithm needs to certify that a function $f$ is positive or to find a counter-example can be upper bounded by 
    \begin{equation}\label{eq:bound_iterations}
        N \leq \frac{\Vol(\bar{\Omega})}{\Vol \left(B\left(\frac{\varepsilon}{2L}\right)\right)},
    \end{equation}
    where $\bar{\Omega}$ is the domain extended by $\frac{\varepsilon}{2L}$ \footnote{The extension of a domain by a quantity $r$ is usually denoted as $\Omega + B\left(r\right)$, where $+$ represents the Minkowski sum \cite{Schneider2013ConvexTheory}.} and $B\left(\frac{\varepsilon}{2L}\right)$ is a ball of radius $\frac{\varepsilon}{2L}$.
\end{theorem}

As observed, $N$ depends on the selected $\varepsilon$ and the Lipschitz constant $L$. Consequently, a larger $\varepsilon$ or a smaller $L$ increases the size of the minimum acceptable balls, reducing the maximum number of iterations required to cover $\Omega$ in the worst-case scenario, as dictated by Eq. \eqref{eq:bound_iterations}

After the LipVor Algorithm has reached the maximum number of iterations $N$ or stops upon verifying the sufficient condition stated in Eq. \eqref{eq:glob_condition}, the function $f$ is certified positive (if $\textit{isPositive} =$ True) or certified not $\varepsilon$-positive (if $\textit{cExamples} \not=$ None or $\textit{cExamples} =$ None and $\textit{isPositive} =$ False). However, as mentioned before, considering $\varepsilon$-positive is not a restrictive constraint as any positive function in a compact domain is $\varepsilon_{f,\Omega}$-positive for a specific $\varepsilon_{f,\Omega}$. Therefore decreasing sufficiently $\varepsilon$, any positive function becomes $\varepsilon$-positive. Consequently, considering $\varepsilon$-positivity is just a formal convention to establish Eq. \eqref{eq:bound_iterations}.

As a final remark, as stated in appendix \ref{sec:ap_bound_iterations}, Eq. \eqref{eq:bound_iterations} can be upper bounded by 
\(
    N \leq \frac{\Vol(\bar{\Omega})}{\Vol \left(B\left( \frac{\varepsilon}{2L}\right)\right)}\leq
    \left(\frac{2\cdot \left(a\cdot L + \varepsilon\right)}{\pi^{\frac{1}{2}}\cdot\varepsilon}\right)^n \Gamma \left(\frac{n}{2}+1\right),
\)
where $\Gamma$ is the gamma function and $a$ is a constant dependent on the size of the domain $\Omega$.

\section{Certification of Partial Monotonicity for ANNs}\label{sec:mono_nn}
\subsection{Partial Monotonicity of an ANN}\label{sub_sec:mono_nn}
As previously mentioned in Section \ref{sec:1}, partial monotonicity of a function $g$ w.r.t. the $r^{th}$ input, with $0\leq r \leq n$, can be formulated as a positivity constraint over the partial derivative $g_{x_r}$. Mathematically, a function $g: \Omega \subseteq \mathbb{R}^n \rightarrow \mathbb{R}$ is strictly increasing (resp. decreasing) partially monotonic w.r.t. the $r^{th}$ input if 
\begin{equation}\label{eq:mono}
          g(x_1,\dots,x_{r},\dots,x_n) < g(x_1,\dots,x'_{r}\dots,x_n) \, , \forall \ x_r < x'_r\,
 \end{equation}
 \(
          (\text{resp.} \; g(x_1,\dots,x_{r},\dots,x_n) > g(x_1,\dots,x'_{r}\dots,x_n)).
 \)

Consequently, a function $g$ will be partially monotonic w.r.t. a set of features $\{x_{i_1},\dots,x_{i_k}\}$ wih $k\leq n$ whenever Eq. \eqref{eq:mono} is verified for each $i_j$ simultaneously with  $j \in \{1,\dots,k\}$.

Therefore, considering a function $g: \Omega \subseteq \mathbb{R}^n \rightarrow \mathbb{R}$ with $g \in C^1(\Omega)$, if $g_{x_r}(x)>0$ (resp. $g_{x_r}(x)<0$) the function $g$ is strictly increasing (resp. decreasing) partially monotonic w.r.t. the $r^{th}$ input feature at the point $x \in \Omega $. Moreover, whenever $g_{x_r}(x) \geq \varepsilon$ ($g_{x_r}(x) \leq -\varepsilon$), the function $g$ is increasing (decreasing) partially $\varepsilon$-monotonic. Consequently, the methodology stated in Section \ref{sec:lipvor} can be applied to certify partial monotonicity of a function $g \in C^{1,1}(\Omega)$ (differentiable and with L-Lipschitz partial derivatives)  by just taking $f := g_{x_r}$.

On the other hand, ANNs, specifically feedforward neural networks, are composed of interconnected layers of neurons that process and transform data. Therefore, an ANN $g: \Omega \subseteq \mathbb{R}^n \rightarrow \mathbb{R}$ with $K$ layers and $n^l$ neurons in the l-layer, with $0 < l \leq K$, can be mathematically described as a composition of point-wise multiplication with non-linear activation functions, such that the output of the $l^{th}$ layer $o^l$ is given by
\begin{align}
    o^0 &= \mathbf{x}, \nonumber \\
    z^l &= o^{l-1} \cdot W^l + b^l \quad \text{for } l = 1, 2, \ldots, K,  \label{eq:z_l}\\\
o^l &= \varphi^l \left( z^l \right), \quad \text{for } l = 1, 2, \ldots, K,  \label{eq:o_l}\\
    y &= g(\mathbf{x}; \mathbf{W}, \mathbf{b}) =  \varphi^K \left( \mathbf o^{K-1} \cdot W^K + b^K \right) \label{eq:ANN_descrip},
\end{align}
where $\mathbf{x}$ represents the input data, $W^l \in \mathbb{R}^{n^{l-1}} \times \mathbb{R}^{n^{l}}$, $b^l \in \mathbb{R}^{n^{l}}$ and $\varphi^l(\cdot)$ are the weights, bias and activation function of the $l^{th}$ layer respectively. 

In particular, ANNs using $C^2 (\Omega)$ activation functions, such as Sigmoid, Tanh, etc., are $C^{1,1}(\Omega)$ functions in a compact domain $\Omega$. Moreover, common non-smooth activations (e.g., ReLU) can be approximated via smooth surrogates like SoftPlus \cite{Dugas2000IncorporatingPricing}, enabling the application of the proposed methodology.
Therefore, knowing an upper bound of the Lipschitz constant of the partial derivatives of an ANN, it is possible to pose the partial monotonicity certification problem as an application of the LipVor Algorithm to the partial derivatives. Consequently, computing an upper bound $\hat{L}$ of the Lipschitz constant $L$ of the partial derivatives of an ANN becomes crucial.

Without loss of generality, in the following methodology, the ANN is going to be considered strictly increasing partially monotonic w.r.t. the $r^{th}$ input. Subsequently, when dealing with multiple monotonic features, the procedure remains the same by simply considering the minimum radius among each of the monotonic features. By using this minimum radius as the ball's radius, we ensure monotonicity is preserved for each monotonic feature. Conversely, if there exists a point $p_j$ where the constraint is not satisfied for any of the monotonic features, a radius $\delta_j=\frac{g_{x_i}(p_j)}{L_j}$ is computed for each unsatisfied constraint and the maximum of these radii is used. This maximum radius represents the largest radius within which some of the monotonicity conditions are not met.
Moreover, the methodology could be extended to vector-valued ANN $g: \mathbb{R}^n \rightarrow \mathbb{R}^m$ considering each of the components from the codomain $g^i, i \in \{1,2,\dots,m\}$.

\subsection{Lipschitz Constant Estimate of the Partial Derivative of an ANN}
As mentioned in Section \ref{sec:1}, the study of estimates of the Lipschitz constant of an ANN has already been conducted, but this was not the case for the Lipschitz constant of the partial derivative of an ANN. This study presents a novel approach to computing an upper bound for such Lipschitz constant.

Recall that if $f$ is an L-Lipschitz function $f \in C^{0,1}(\Omega_C)$ in a compact convex domain $\Omega_C$ then the Lipschitz constant of $f$ in $\Omega_C$ is the maximum norm of its gradient. This result, presented in the following proposition, could be stated more generally; however, for the sake of simplicity, the result is restricted to differentiable functions with codomain $\mathbb{R}$. (cf. Thm. 3.1.6
\cite[Rademacher]{Federer1996GeometricTheory}).
\begin{proposition}\label{prop:lip_gradient}
    Let $\Omega_C \subseteq\mathbb{R}^n$ a compact convex domain and $f:\Omega_C  \rightarrow \mathbb{R}$ a $C^1(\Omega_C)$ function. Then
    \(
        L  = \sup_{x\in \Omega_C}\norm{\nabla f(x)},
    \)
    where $\nabla f(x)$ is the gradient of the function $f$.
\end{proposition}

Considering now an ANN $g:\Omega_C \subseteq \mathbb{R}^n\rightarrow \mathbb{R}$ with $\Omega_C$ a compact convex domain and $g \in C^{1,1}(\Omega)$, then, applying Proposition \ref{prop:lip_gradient} to the $r^{th}$ partial derivative $g_{x_r}$ of the ANN, the Lipschitz constant is obtained as the maximum norm of $\tilde{H}_g$ where 
\(
    \tilde{H}_g = \left(\frac{\partial g}{\partial x_r \partial x_j}\right)_{j},   \forall \; 1\leq j \leq n,
\)
is the Hessian of the output of $g$ w.r.t. the $r^{th}$ input. 
Hence by Proposition \ref{prop:lip_gradient}, 
\(
    L = \sup_{x\in \Omega} \norm{\nabla g_{x_r}(x)} =  \sup_{x\in \Omega} \norm{\tilde{H}_g(x)}.
\)
Consequently, an upper bound of $\norm{\tilde{H}_g}$ establishes an upper estimate $\hat{L}$ of the Lipschitz constant $L$ of the $r^{th}$ partial derivative of an ANN. 

To generate an upper bound of $\norm{\tilde{H}_g}$, it is going to be used that an ANN with K-layers can be described as a composition of activation functions and linear transformations (Eq. \eqref{eq:z_l}-\eqref{eq:ANN_descrip}). Therefore, considering  
\begin{equation}\label{eq:H_0_l}
 H_0^l = \left(\frac{\partial^2 o^l_{k}}{\partial x_i \partial x_j }\right)_{i,j,k},\, \forall \, 1\leq i,j \leq n^0, 1 \leq k \leq n^l,
\end{equation}
the Hessian of the $l^{th}$ layer w.r.t. the input layer (layer 0), then, by Eq. \eqref{eq:ANN_descrip},
\(
    \tilde{H}_g = \tilde{H}_0^K = \left(\frac{\partial^2 o^K}{\partial x_r \partial x_j }\right)_{j},\, \forall \, 1\leq j \leq n^0.
\)
Therefore, the maximum norm of the Hessian of the $k^{th}$ layer of the ANN w.r.t. the $r^{th}$ input ($L_{x_r}^K$) can be recursively upper bounded using the weights of the preceding layers $W^l$, the Hessian of the $K^{th}$ layer
\(
    H_K^K = \left(\frac{\partial^2 o^K}{\partial z_i \partial z_j }\right)_{i,j}, 1\leq i,j \leq n^{K-1}
\)
with respect to $z^l$ (Eq. \eqref{eq:z_l}) and the upper bound of the $K-1$ layer $\hat{L}_{x_r}^{K-1}$ as follows.

\begin{theorem}\label{teo:lip_acot}
        Let $g:\Omega \subseteq\mathbb{R}^n \rightarrow \mathbb{R}$ be an ANN with K-layers, then the Lipschitz constant $ L_{x_r}^K$ of the $r^{th}$ partial derivative $g_{x_r}$ of an ANN verifies that
    \begin{multline}\label{eq:lip_up_bound}
        L_{x_r}^K \leq \hat{L}_{x_r}^K := \max|a_k^K|\cdot\lVert W_{1j}^1\rVert \cdot\lVert W^1\rVert\cdot\lVert W^2\rVert ^2 \dots \lVert W^K\rVert ^2\\ +\hat{L}_{x_r}^{K-1} 
        \cdot\lVert W^K \rVert,
    \end{multline}
    where $W^k$ is the weight matrix of the $k^{th}$ layer, $W_{1j}^1$ is the first row of the weight matrix $W^1$, $\hat{L}_{x_r}^{K-1}$ is the upper bound of Lipschitz constant of the $r^{th}$ partial derivative of the $K-1$ layer of the ANN and $a_k^K = \|H_K^K\|_{\infty}:= \max_{1\leq i\leq n}\left(\sum_{j=1}^n \left(H_K^K\right)_{i,j}\right)$.
\end{theorem}
A complete proof of Theorem \ref{teo:lip_acot} can be found in appendix \ref{lips_appendix}.
 
Finally, it is important to acknowledge that for the most commonly used activation functions in neural networks, the second derivative is upper bounded by 1. For example, in the case of the sigmoid activation function, the second derivative is given by $\sigma''(x)=\sigma'(x)(1-2\cdot \sigma(x))$, which is clearly upper bounded by 1. Moreover, note that the aforementioned theorem is valid for vector-valued ANNs $g:\mathbb{R}^n \rightarrow \mathbb{R}^m$ as  Eq. \eqref{eq:lip_up_bound} is recursively obtained using the upper bound of the intermediate layers which are vector-valued functions with codomain's dimension the number of neurons in the intermediate layer. 
\section{Extension: Training Certified Monotonic Neural Networks}\label{sec:mono_training}

This section introduces a methodology to train unconstrained, probably partial $\varepsilon$-monotonic ANNs that can later be certified using the \mbox{\hyperref[alg:alg1]{LipVor Algorithm}}. The approach employs a modified training loss, similar to \cite{Monteiro2022MonotonicityClassification}, where a penalization term enforces an $\varepsilon$-monotonic relation on the training data. By continuity, this is expected to ensure partial monotonicity in neighborhoods of the training points. However, as noted in \cite{Monteiro2022MonotonicityClassification}, monotonicity is guaranteed only near penalized regions, with no assurance over the whole domain $\Omega$. Therefore, after training, the LipVor Algorithm is applied to certify partial monotonicity. Hence, this methodology enables certification without requiring constrained architectures.

Considering an ANN $g:\Omega \subseteq\mathbb{R}^n \rightarrow \mathbb{R}$, to find the optimum ANN's parameters, an optimization procedure to minimize a loss function is followed. The loss function is intended to measure the difference between the real output values $y$ and the predicted $\hat{y}$ ones.  Therefore, the loss function of the ANN can be stated as
\(
\mathcal{L}(\mathbf{W}, \mathbf{b}; \mathcal{D}) = \sum_{(\mathbf{x}, \mathbf{y}) \in \mathcal{D}} \ell\left(y, g(\mathbf{x}; \mathbf{W}, \mathbf{b})\right), 
\)
where $\mathbf{W}$ and $\mathbf{b}$ denote the weights and biases of the neural network, $\mathcal{D}$ is the dataset containing input-output pairs $(\mathbf{x}, y)$ and $\ell$ is the loss function that measures the discrepancy between the predicted output and the actual output. For regression problems, the loss function is usually the sum of squared errors, while for classification problems, the maximum likelihood.


During optimization, the ANN is trained via gradient descent, with gradients efficiently computed by backpropagation. To enforce a partial $\varepsilon$-monotonic relation w.r.t. the $r^{th}$ input, a regularization term is added to the loss, penalizing violations of the constraint. Consequently, the modified loss is given as follows
\(
\tilde{\mathcal{L}}(\mathbf{W}, \mathbf{b}; \mathcal{D}) = \mathcal{L}(\mathbf{W}, \mathbf{b}; \mathcal{D}) + \lambda \Omega_{g,r}(\mathcal{D},\varepsilon),
\)
where $\mathcal{L}$ is the original loss, $\mathcal{D}$ the training dataset, and $\Omega_{g,r}$ the $\varepsilon$-monotonic penalization based on the $r^{th}$ column of the Jacobian matrix. Moreover, the coefficient $\lambda$ controls the strength of the regularization.

To compute the penalization term, the sum of the values of the $r^{th}$ column where the partial $\varepsilon$-monotonicity constraint is not followed is computed. Therefore, if the relation is increasing (decreasing) $\varepsilon$-monotonic, then the rectified values of the negative (positive) sum of the Jacobian matrix are summed up. Consequently, the regularization term is computed as 
\begin{equation}\label{eq:pen_term}
    \Omega_{g,r}(\mathcal{\bar{D}},\varepsilon) = \sum_{x \in \mathcal{\bar{D}}} \left[\max\left(0,-J_{r}(\mathbf{x})+ \varepsilon\right)\right],
\end{equation}
where  $J_{r}$ represents the $r^{th}$ column of the Jacobian matrix evaluated at $\mathbf{x}$ $\left(\text{i.e.,} \quad J_{r} = \frac{\partial g}{\partial x_r}\right)$. Considering this strategy, the parameters' optimization will be guided towards ensuring that the partial $\varepsilon$-monotonicity constraint is fulfilled for all data samples in $\mathcal{D}$. Notice that in practice, it is not necessary to compute the full Hessian in the optimization step but instead just the gradient of the Jacobian for the partial derivatives for which the monotone restriction is imposed. This step could be done efficiently using tensorial derivatives, as proposed in \cite{Gonzalo2023ExplainableNetworks, Pizarroso2022NeuralSens:Networks}

It is worth mentioning that since \mbox{\(\varepsilon\)} governs the minimum acceptable slope of the monotonicity constraint, a very large \mbox{\(\varepsilon\)} might result in a steeper slope than necessary, causing unnecessary penalization that could harm model performance. Hence, we suggest tuning epsilon based on the specific characteristics of the problem, starting with a small positive value if no prior knowledge is available. Moreover, if certification is not achieved within a reasonable time, increasing $\varepsilon$ slightly and retraining can help the algorithm converge more efficiently while still enforcing monotonicity.

After training, the LipVor Algorithm is executed to certify partial monotonicity. If the ANN is not partially $\varepsilon$-monotonic and counter-examples are found, these points form an external dataset where the penalization is also enforced. Thus, the penalization targets subdomains where monotonicity fails, improving over \cite{Monteiro2022MonotonicityClassification}, which uses random points. Fine-tuning the ANN with both training and external data is expected to yield certification by LipVor. Therefore, this process can be applied iteratively until convergence to a partial monotonic ANN.

Besides, to avoid overfitting, we apply early stopping after $N_p$ consecutive epochs without validation loss improvement, where $N_p$ denotes patience. Early stopping is only triggered when the penalization loss of the partial $\varepsilon$-monotonicity constraint reaches 0, indicating the ANN is partial $\varepsilon$-monotonic on the training data and justifying attempting to verify in $\Omega$.

\section{Computational Implementation Details}\label{sec:implementation_details}

In this section, we present the computational implementation details of the algorithm as well as other aspects of the computational complexity. First of all, the proposed code for this study was developed using the Python programming language and the PyTorch framework \mbox{\cite{Paszke2019PyTorch:Library}}. Moreover, for the efficient computation of the ANN's partial derivatives in the LipVor algorithm, the NeuralSens package \mbox{\cite{Pizarroso2022NeuralSens:Networks}} was utilized. Additionally, the algorithm and the results obtained in this paper are available in the public repository at \mbox{\href{https://github.com/alejandropolo/LipVor}{https://github.com/alejandropolo/LipVor}}.


For the implementation of the Voronoi, we have leveraged Julia’s HighVoronoi.jl library \cite{HeidaHomeHighVoronoi.jl}, based on \cite{Heida2023OnDiagrams}, for accelerated Voronoi computations. The Voronoi-based verification step, which forms the core of our approach, has a worst-case complexity of $O(n^{2-1/d})$ for $n$ points in a $ d$-dimensional space, though empirical estimates suggest a more favourable complexity closer to $O(n\cdot \log(n))$ \cite{Heida2023OnDiagrams}. Therefore, if $ N$ steps are taken to conclude the LipVor algorithm, the total complexity of the algorithm in the worst-case scenario would be $ O(N \cdot n^{2-1/d}) $. Additionally, since the proposed approach involves verifying whether the Voronoi regions sufficiently cover the input space, it naturally lends itself to parallelization, as the input domain can be divided into smaller $ n $-dimensional cuboids, each processed independently on separate cores. Therefore, the domain can be verified if each of the cuboids is verified. This parallelization strategy, which avoids costly synchronization steps, significantly enhances efficiency.

Furthermore, as previously discussed, we utilize the $L^2$ norm for all norms in this study. 
Lastly, we selected Voronoi diagrams over other tessellation methods because they provide a straightforward mathematical criterion, as outlined in Lemma \mbox{\ref{lem:covering}}, to verify whether the union of open balls fully covers an $n$-dimensional space.

All computational experiments were conducted on a machine equipped with an Intel(R) Core(TM) i7-9750H CPU @ 2.60GHz, featuring 6 physical cores and 12 logical threads, along with 16 GB of RAM.

\section{Case Studies}\label{sec:case_studies}
In the following section, three case studies where partial monotonicity constraints should be enforced are presented. In the first case, an ANN is used to estimate the heat transfer of a 1D bar. As the input domain is 2 dimensional, we can visualize the Voronoi expansion and the partial monotonicity certification using the \hyperref[alg:alg1]{LipVor Algorithm}. Second, a four-dimensional dataset illustrates fairness by ensuring positive attributes are not penalized, highlighting monotonicity’s role in trustworthy predictions. Finally, monotonicity constraints are applied to the AutoMPG dataset, demonstrating the use of LipVor in higher-dimensional spaces. Together, these three experiments constitute real-world scenarios where partial monotonicity aligns with physics, trustworthiness, or expected behavior, respectively.


For model training and evaluation, the dataset is split into training, validation, and testing. First, $20\%$ of the samples are reserved for testing to provide an unbiased assessment. The remaining $80\%$ is further divided $80/20$ into training (model fitting) and validation (hyperparameter tuning and overfitting control). The ANN architecture and hyperparameters, such as learning rate and weight decay, were selected via grid search. The specific parameter ranges are available in the accompanying \mbox{\href{https://github.com/alejandropolo/LipVor}{GitHub repository}}.



\subsection{Case Study: Heat Equation}
Accurately modeling mechanical behavior from observed data is essential for predictive maintenance and digital twin development \cite{Wen2022RecentPerspective}, but capturing complex dynamics remains challenging \cite{Bofill2023ExploringOpportunities}. Traditional methods often fall short, motivating the use of advanced techniques like ANNs \cite{Zhang2019Data-DrivenSurvey}. Yet, ensuring that ANNs comply with physical laws is crucial, as it improves predictive accuracy and preserves fundamental principles \cite{Raissi2019Physics-informedEquations}. To illustrate this, in this case study we consider heat distribution in a 1D rod, showing how a physical property of the system can be encoded as a partial monotonicity constraint.

The evolution of heat in such a rod is governed by the heat equation, a fundamental partial differential equation (PDE) in mathematical physics. In its simpler form, it can be expressed as $\frac{\partial u}{\partial t} = k \cdot \frac{\partial^2 u}{\partial x^2}$, where $u(x,t)$ denotes the temperature as a function of space ($x$) and time ($t$), and $k$ is the thermal diffusivity constant.

In this study, Dirichlet boundary conditions (BC) are imposed to specify the temperature behavior at the rod’s ends, where the boundary temperature increases linearly with time, representing constant heat addition or removal. Hence, the heat equation with the aforementioned time-dependent Dirichlet BC can be described as
\begin{equation}\label{eq:heat_eq}
\left\{
\begin{aligned}
& \frac{\partial u}{\partial t} = k \cdot \frac{\partial^2 u}{\partial x^2}, \quad 0 < x < L,\ t > 0, \\
& u(0,t) = t,\quad u(L,t) = t,\quad u(x,0) = 0.
\end{aligned}
\right.
\end{equation}
Under these BC, the solution of \eqref{eq:heat_eq} is partially monotonic in $t$,\footnote{The linear boundary conditions and the diffusive nature of \eqref{eq:heat_eq} guarantee that $u(x,t)$ increases with $t$.} so enforcing monotonicity matches the expected physical behavior.

Moreover, since collecting physical data is often costly and noisy due to sensor inaccuracies \cite{Bofill2023ExploringOpportunities}, we generated a synthetic dataset to replicate real-world conditions with limited, noisy samples. Specifically, we sampled 30 random measurements of $u(x,t)$ from simulations of \eqref{eq:heat_eq} and added Gaussian noise $\mathcal{N}(0,0.02)$ to each point (Fig.~\ref{fig:heat_data}).

Multiple experiments were conducted using different random seeds, with each seed generating a distinct dataset. For most generated datasets, the unconstrained ANN was successfully certified as partially monotonic after enforcing the penalization term given by Eq. \mbox{\ref{eq:pen_term}}. However, to illustrate a worst-case scenario in which the penalization term alone was insufficient to ensure monotonicity across the entire domain, this case study focuses on a specific corner case where the unconstrained ANN did not converge to a partial monotonic ANN. As illustrated in Figure \mbox{\ref{fig:init_ann}}-(a), the partial derivative of the trained ANN w.r.t. input $t$ is not positive in the whole $\Omega$. Specifically, whenever $t \approx 0$ and $x \in [0.5,0.6]$ the partial derivative is negative. Therefore, as increasing partial monotonicity is equivalent to having positive derivatives, the ANN is not partially monotonic in that subdomain. Consequently, this case study provides an example of an unconstrained ANN that fails to achieve partial monotonicity despite the application of the penalization term.

\begin{figure}[!htbp]
\centering
\includegraphics[width=\linewidth]{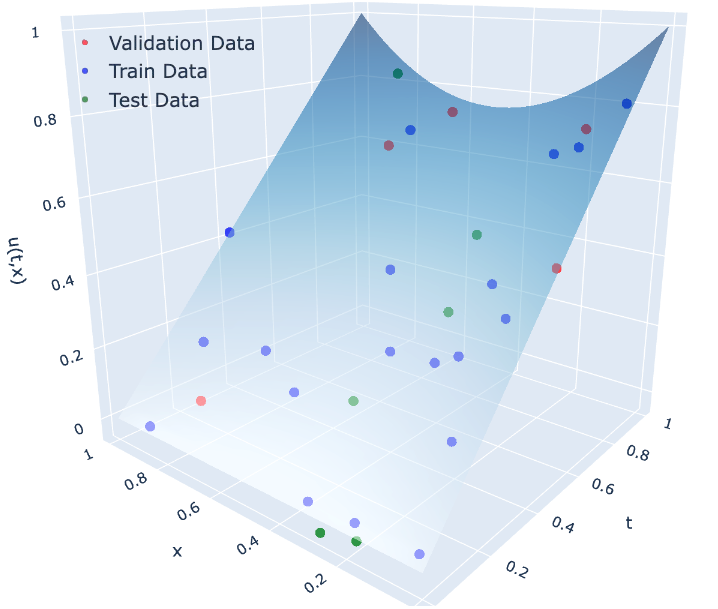}
\caption{Surface of the solution of the heat equation \eqref{eq:heat_eq} and the training (blue), validation (red) and test (green) datasets obtained from the solution with added noise. }
\label{fig:heat_data}
\end{figure}

Considering the ANN's architecture, given the limited number of training samples, a feedforward ANN with 1 hidden layer and 10 neurons with a hyperbolic tangent (tanh) activation function has been considered. Moreover, L-BFGS has been selected as the optimizer with a learning rate of $0.01$. Besides, a penalization of $\lambda = 0.1$ has been imposed to guide the training towards a $\varepsilon$-monotonic solution taking $\varepsilon = 0.1$. The maximum number of epochs for the training process was fixed at 5000 with a patience of $N_p=1000$ epochs for the early stopping.

After the training process is completed, the LipVor Algorithm is used to determine if the ANN is partially $\varepsilon$-monotonic w.r.t. $t$ in $\Omega$. As mentioned before, the partial derivative of the trained ANN w.r.t. input $t$ is not positive in the whole $\Omega$ (Figure \ref{fig:init_ann}-(a)). Therefore, the LipVor Algorithm is expected to be able to effectively find counter-examples in the aforementioned region. 

Following the bound proposed in Theorem \ref{teo:lip_acot}, a Lipschitz estimate of $\hat{L} = 17.14$ is obtained.  Moreover, to start the LipVor Algorithm, ten points from the training dataset are selected to initialize the algorithm, and a maximum number of 800 iterations is fixed (Figure \ref{fig:init_ann}-(b)). As observed in Figure \ref{fig:init_ann}-(c), after 21 iterations the LipVor Algorithm has effectively detected the first counter-example. Moreover, after reaching the maximum number of iterations (Figure \ref{fig:init_ann}-(d)), 77 counter-examples of partial $\varepsilon$-monotonicity have been detected in contrast with the certified area which corresponds to the $71.2\%$ of $\Omega$.  

\begin{figure*}[!htbp]\
\centering     
\subfigure[]{
\includegraphics[width=40mm]{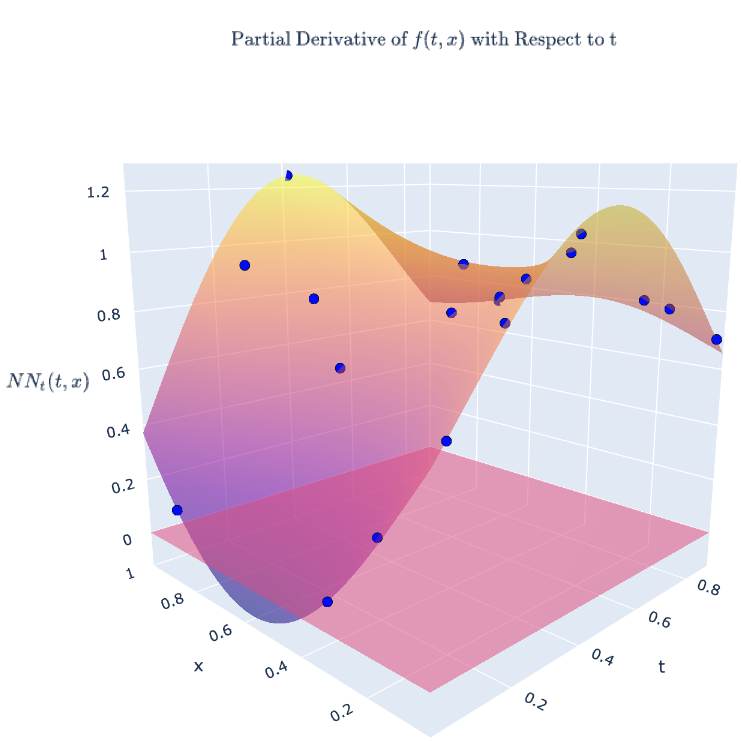}
}
\subfigure[]{
\includegraphics[width=40mm]{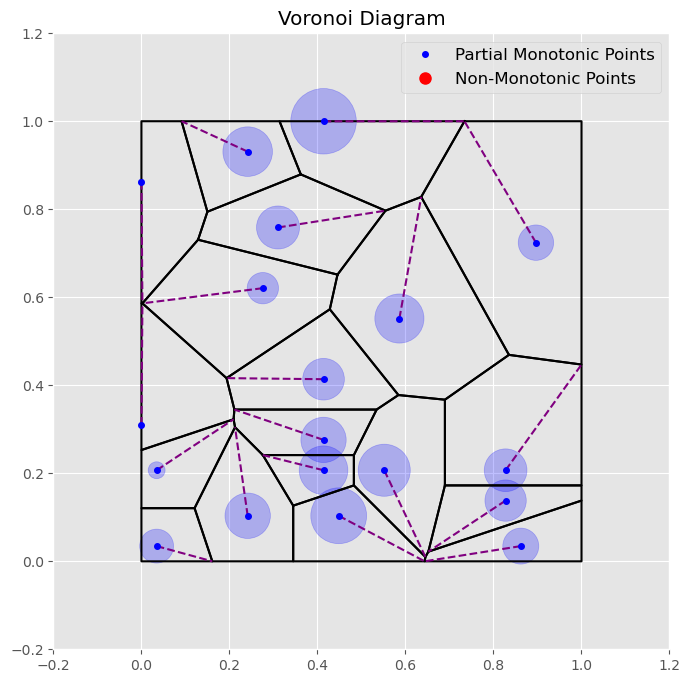}
}
\subfigure[]{
\includegraphics[width=40mm]{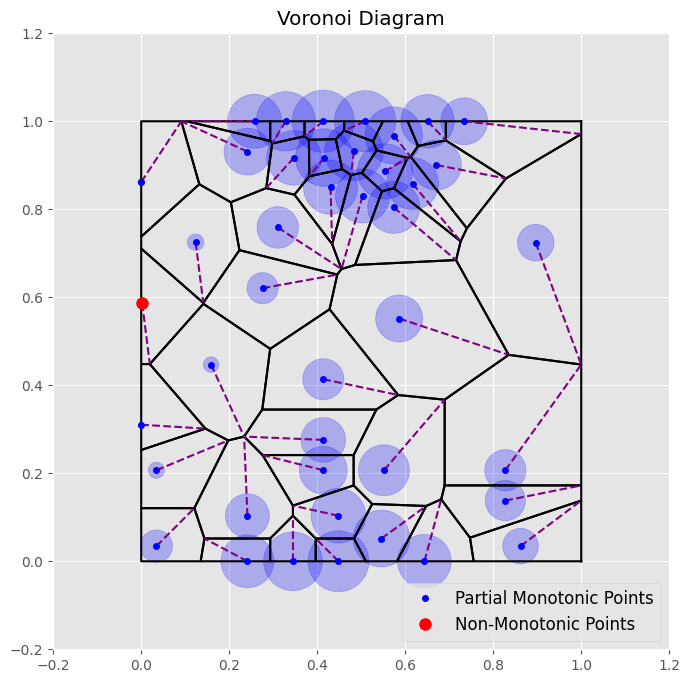}
}
\subfigure[]{
\includegraphics[width=40mm]{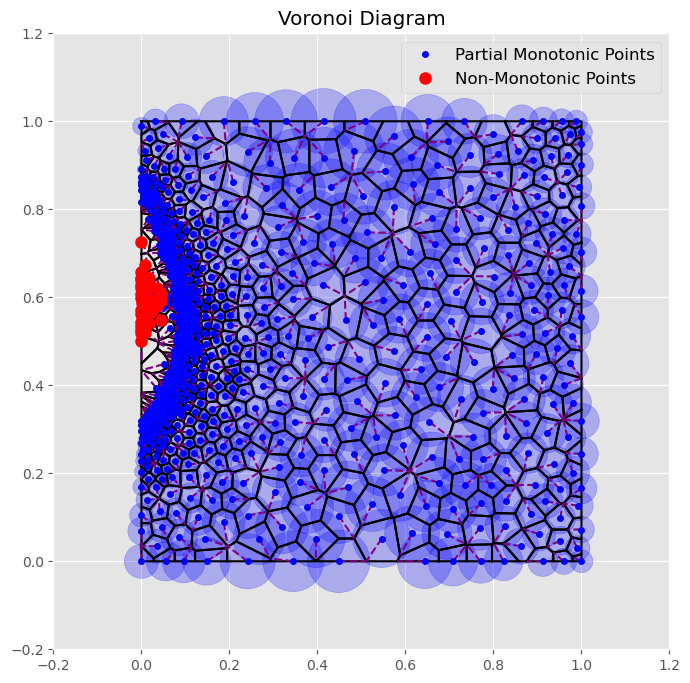}
}
\caption{Evolution of the Voronoi diagram generated by the LipVor Algorithm. (a) Surface plot of the partial derivative of the ANN output w.r.t the input $t$ across the domain. A horizontal plane at $z=0$ is included to help identify regions where the partial derivative is positive or negative.(b) Initialization of the Voronoi diagram using the training set. (c) Voronoi diagram expansion when the first counter-examples is detected by the LipVor Algorithm. (d) Final iteration, after LipVor has reached the maximum number of steps, showing the partial monotonic subdomain and the found counter-examples. }
\label{fig:init_ann}
\end{figure*}

Once the LipVor Algorithm has proved that the initial ANN is not partially $\varepsilon$-monotonic, based on the found counter-examples, the ANN is fine-tuned. The initial ANN has a training mean absolute error (MAE) of $8.81 \times 10^{-3}$ and a test MAE of $1.64 \times 10^{-2}$, with corresponding coefficient of determination $(R^2)$ \cite{Rawlings1998AppliedStatistics} values of 0.9985 and 0.9935. The training $R^2$ value indicates that the model explained approximately $99.85\%$ of the variance in the training data, reflecting an excellent fit, while the test $R^2$ of $0.9935$ still indicates strong performance on unseen data.

After fine-tuning, the ANN showed a training MAE of $1.23 \times 10^{-2}$ and a test MAE of $1.67 \times 10^{-2}$. The training $R^2$ decreased to $0.9967$ but remained high, showing that $99.67\%$ of the training variance is explained, while the test $R^2$ rose to $0.9953$. The detailed results are presented in Table \mbox{\ref{tab:Heat_ANN}}

\begin{table*}[!htbp]
\caption{Training, validation, and test MAE and $R^2$ results for the initial and fine-tuned ANN in case study A\label{tab:Heat_ANN}}
\centering
\begin{tabular}{ | m{1.5cm} | m{1.5cm} | m{1.5cm} | m{1.5cm} | m{1.5cm} | m{1.5cm} | m{1.5cm} | }
    \cline{2-7}
    \multicolumn{1}{c|}{} & \multicolumn{3}{c|}{Initial ANN} & \multicolumn{3}{c|}{Fine-Tuned ANN} \\
    \cline{2-7}
    \multicolumn{1}{c|}{} & Training & Validation & Test & Training & Validation & Test \\
    \hline
    MAE  & $8.81 \times 10^{-3}$ & $2.56 \times 10^{-2}$ & $1.64 \times 10^{-2}$ & $1.23 \times 10^{-2}$ & $3.05 \times 10^{-2}$ & $1.67 \times 10^{-2}$  \\
    \hline
    $R^2$  & 0.9985 & 0.9861 & 0.9935  & 0.9967 & 0.9801 & 0.9953 \\
    \hline
\end{tabular}
\end{table*}

Once the training is completed, the LipVor Algorithm is executed again to certify if the ANN is now partially monotonic. As observed in Figure \ref{fig:ft_ann}, the partial derivative of the ANN w.r.t. the $t$ input is positive in $\Omega$ and the LipVor Algorithm converges after 696 iterations, with the total computation time being 123 seconds\footnote{For visualization purposes in 2D, the Voronoi diagram in this example was generated using \texttt{scipy.spatial.Voronoi} from Python's SciPy library instead of the optimized Julia-based implementation.}. Therefore, the fine-tuned ANN is certified partially monotonic in $\Omega$.

\begin{figure}[!htbp]
\centering     
\subfigure[]{
\includegraphics[width=40mm]{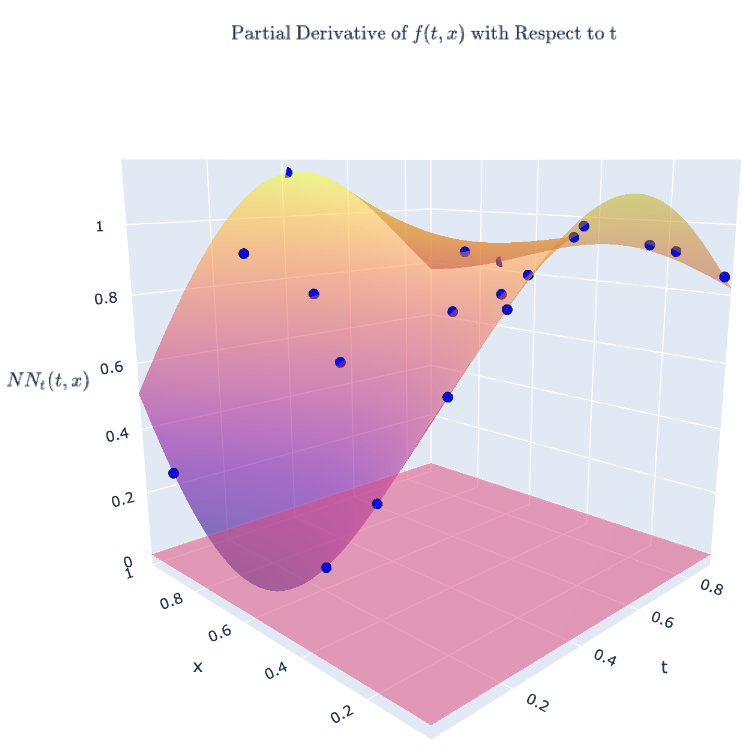}
}
\subfigure[]{
\includegraphics[width=40mm]{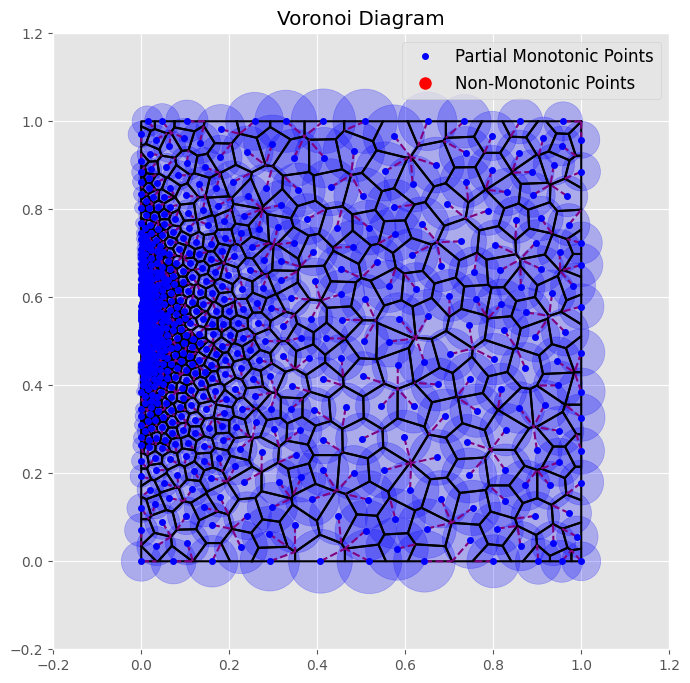}
}
\caption{
Partial monotonicity verification of the ANN using the LipVor Algorithm. (a) Surface plot of the partial derivative of the ANN output w.r.t. input $t$. (b) Voronoi diagram after certifying partial monotonicity.}
\label{fig:ft_ann}
\end{figure}

\subsection{Case Study: Trustworthy Monotonic Predictions for the ESL dataset}

The Employee Selection (ESL) dataset \cite{Ben-David1989LearningConcepts} comprises profiles of candidates applying for specific industrial roles. The four input variables are scores, from 0 up to 9, assigned by expert psychologists based on the psychometric test outcomes and candidate interviews. The output presents an overall score on an ordinal scale from 1 to 9, indicating the extent to which each candidate is suitable for the job. As stated in \cite{Cano2019MonotonicSets}, the ESL dataset is one of the benchmarks in the literature on monotonic datasets. Each four variables are monotonically increasing as a better performance in a psychometric test should be reflected in an increased overall score. Therefore, a model violating the monotonicity constraint would be generating unfair predictions by penalizing a candidate for achieving a higher score on a psychometric test, which is intended to measure their suitability for the job.

The dataset is comprised of four input variables and a total of 488 instances. 
As mentioned before, the four input variables have values ranging in $[0,9]$ and the output variable is comprised in the interval $[1,9]$. Therefore, the input and the output are min-max scaled to transform each variable to $[0,1]$ so the inputs domain is $\Omega = [0,1]^4 \subset \mathbb{R}^4$ and the output domain is $[0,1]$. 

As previously highlighted, one concern of the proposed methodology is the number of executions required to certify the ANN's partial monotonicity. As shown in Proposition \mbox{\ref{prop:rad_mono}}, the radius of certified partial monotonicity is inversely proportional to the Lipschitz constant of the ANN, meaning that a larger Lipschitz constant could reduce the radius and increase the difficulty of ensuring monotonicity. To mitigate this limitation, the inputs and outputs are min-max scaled to the $[0,1]$ range. This normalization, combined with the application of weight decay penalization, helps to constrain the values of the weight matrices within the $[0,1]$ interval, reducing their $L^2$ norm and, consequently, the Lipschitz constant upper bound (Eq. \mbox{\ref{eq:lip_up_bound}}). 

An ANN with an architecture of 2 hidden layers with 5 neurons in each layer is trained. The activation function selected for each neuron is tanh. In this case, an Adam optimizer has been chosen with a learning rate of $0.001$. Moreover, to regularise the ANN, it has been considered a weight decay of $0.005$. The proposed strength of the penalization $\lambda$ is $0.1$ and the ANN is trained considering  $\varepsilon$-monotonicity with $\varepsilon =0.1$. The training process was set to run for a maximum of 5000 epochs, with a patience $N_p=1000$ epochs.  The training process ended after 3987 epochs obtaining the results presented in Table \ref{tab:ESL_ANN} in training, validation and test sets. 

As observed, these results suggest that the model achieved good performance, with the validation and test results reflecting a strong ability to generalize. In particular, an $R^2$ of 0.88129 on the test set indicates that the model explains roughly $88.13\%$ of the variability in the test data, demonstrating a strong fit of the trained model.

\begin{table}[!htbp]
\caption{Training, validation, and test MAE and $R^2$ results for the trained ANN in case study B \label{tab:ESL_ANN}}
\centering
\begin{tabular}{ | m{5em} | m{1.5cm}| m{1.5cm} | m{1.5cm}|}
\hline
& Training & Validation &Test \\
\hline 
MAE  & 0.05594 & 0.04285 & 0.04830  \\
\hline
$R^2$  & 0.82228 & 0.89053 & 0.88129  \\
\hline
\end{tabular}
\end{table}

After the training comes to an end, the LipVor Algorithm is used to check whether the trained ANN is monotonic w.r.t. the four input variables. In this case, the Lipschitz estimate is $\hat{L} = 0.88$ and the LipVor Algorithm is computed starting from 10 random points chosen from the training set. After 548 iterations, executed in 77 seconds, the LipVor Algorithm certifies that the whole input space $\Omega$ is covered and, therefore, the trained ANN is provably monotonic w.r.t. the input variables.

\subsection{Case Study: AutoMPG Dataset}
The Auto MPG dataset \cite{Quinlan1993AutoMPG}, a standard benchmark in monotonic regression \cite{Cano2019MonotonicSets}, encompasses various attributes of automobiles, aiming to predict their fuel efficiency in miles per gallon (MPG). The dataset comprises 398 instances with input features including cylinders, displacement, horsepower, weight, acceleration, model year, and origin, while the output is the MPG. In our experiments, the input and output variables are normalized to the $[0,1]$ interval to ensure consistency and improve training performance.


An ANN with a single hidden layer containing 10 neurons is trained. The activation function selected for each neuron is sigmoid. In this case, an Adam optimizer has been chosen with a learning rate of $0.01$. Moreover, to regularize the ANN, a weight decay of $0.0007$ has been applied. The proposed strength of the penalization $\lambda$ is $0.1$, and the ANN is trained considering $\varepsilon$-monotonicity with $\varepsilon = 0.2$. The training process was set to run for a maximum of 10,000 epochs, with a patience $N_p = 1000$ epochs for early stopping. The training process concluded after 4,881 epochs due to the early stopping criteria being met. The obtained results, presented in Table \ref{tab:AutoMPG_ANN}, show a MAE of 2.19 on the training set, 2.01 on the validation set, and 2.33 on the test set.

\begin{table}[!htbp]  
\caption{Training, validation, and test MAE and $R^2$ results for the trained ANN in case study C (unscaled values) \label{tab:AutoMPG_ANN}}
\centering
\begin{tabular}{ | m{5em} | m{1.5cm}| m{1.5cm} | m{1.5cm} | }
\hline
& Training & Validation & Test \\
\hline 
MAE  & 2.18917 &  2.00626  & 2.33156 \\
\hline
$R^2$  & 0.85431 & 0.87267  &  0.81531 \\
\hline
\end{tabular}%
\end{table}


As observed, these results suggest that the model achieved solid performance, with both the validation and test results demonstrating the model's ability to generalize well. Specifically, the $R^2$ value of 0.81531 on the test set indicates that the model explains approximately $81.53\%$ of the variance in the test data, highlighting a strong fit of the trained model. Additionally, the MAE values show good accuracy in the model’s predictions. Moreover, it is important to note that these results were obtained after unscaling and transforming the output variable back to its original scale.

After training concludes, the LipVor Algorithm is applied to verify whether the trained ANN is decreasingly monotonic with respect to displacement, horsepower and weight. Using a Lipschitz estimate of $\hat{L} = 0.89$, the algorithm is initialized with 1000 random points sampled within the domain $\Omega$. To accelerate verification, the domain is divided into multiple subdomains, each independently verified in parallel. The entire process was completed in approximately 124 seconds, successfully certifying that the ANN is provably monotonic over the input space.
\section{Conclusions}\label{sec:conclusion}
In this article, we propose a novel algorithm (LipVor) that leverages the Lipschitzianity of a black-box model to certify positivity in the whole input domain based on a finite set of positively evaluated points. In particular, as the partial monotonicity of an ANN can be stated as a positivity condition of the ANN's partial derivatives, it is possible to apply the LipVor Algorithm to mathematically certify if an ANN is partially monotonic or find counter-examples. To do so, an upper bound of the Lipschitz constant of the partial derivatives of an ANN is also presented.

The results show that the LipVor Algorithm enables applications in sectors such as banking, where regulators demand trustworthy predictions under partial monotonicity constraints. Beyond monotonicity, other properties like convexity can also be framed as positivity constraints and certified with LipVor. In addition, the methodology may further extend to architectures such as convolutional or normalization layers, while future work includes exploring recurrent networks and transformers. Moreover, we acknowledge that enhancing the computational efficiency of Voronoi diagrams through alternative algorithms is a promising research direction. Since the method may be less efficient in high-dimensional settings, exploring other approaches for space coverage verification is also relevant. Besides, alternative tessellations, such as cubic honeycomb or centroidal Voronoi tessellation, offer further possibilities. Finally, efficiency could also improve by restricting furthest-vertex calculations to newly added points and their neighboring cells, or by using data structures like kd-trees for faster searches.

\appendices
\section{Proof of Theorem \ref{teo:glob_condition}}\label{sec:ap_glob_condition}

This appendix presents a proof of Theorem \ref{teo:glob_condition}, that states a sufficient condition for an L-Lipschitz function $f: \Omega \subseteq \mathbb{R}^n \rightarrow \mathbb{R}$ to be certified positive in $\Omega$. First of all, let us present a technical lemma that will allow us to check if a union of balls $\bigcup_{i \in I}B(p_i, \delta_i)$ covers $\Omega$.

\begin{lemma}\label{lem:covering}
    Let $\Omega \subset \mathbb{R}^n$ be a compact domain and $\mathcal{P}= \{p_1,p_2,\dots,p_k\}$ a set of points  defining a Voronoi diagram with cells $(R_i)_{i\in I}$, with $I=\{1,\dots,k\}$, if $\forall \, p_j \in \mathcal{P}$ it is verified that 
    \(
        \max_{x \in R_j} d(x,p_j) < \delta_j , 
    \)
    for some $\delta_j$ with $j \in I$, then 
    \(
          \Omega\subseteq\bigcup_{i \in I} B(p_i, \delta_i) .
    \)
\end{lemma}

\begin{proof}
Let us prove the lemma by contradiction. Let us suppose that there exists $x^* \in \Omega$ such that $x^* \not\in \bigcup_{i \in I} B(p_i, \delta_i)$. By definition of the Voronoi diagram in a compact space $\Omega$, there exists at least one $p_j \in \mathcal{P}$ such that $x^* \in R_j$. On the other hand, as $x^* \not\in  \bigcup_{i \in I} B(p_i, \delta_i)$, then $d(x^*,p_j)>\delta_j$. Therefore
\(
d(x^*,p_j)>\delta_j> \max_{x \in R_j} d(x,p_j),
\)
which is a contradiction with the initial supposition. 
\end{proof}

Recall that Proposition \ref{prop:rad_mono} states that if $f: \Omega \subset \mathbb{R}^n \rightarrow \mathbb{R}$ is an L-Lipschitz function such that $f(x_0)>0$, with $x_0 \in \Omega$, then there exists a radius $\delta_0 = \frac{f(x_0)}{L}$ of certified positivity verifying that $f(x)>0, \; \forall \; x \in B(x_0,\delta_0)$. Therefore, given a set of points $\mathcal{P} = \{p_1, p_2, \dots, p_k\}$ and an L-Lipschitz function $f$, the previous lemma can be used to state a sufficient condition for $f$ to be certified positive based on the evaluations at $\mathcal{P}$. Specifically, it suffices to check if the furthest point from each Voronoi cell to the point defining the cell, is smaller than the radius of certified positivity given by Proposition \ref{prop:rad_mono}.

\begin{theorem}\label{teo:ap_glob_condition}
    Let $\Omega \subset \mathbb{R}^n$ be a compact domain, $\mathcal{P} = \{p_1, p_2, \dots, p_k\}$ a set of points contained in $\Omega$ and $V(\mathcal{P})=(R_j)_{j\in I}$ be the Voronoi diagram of $\mathcal{P}$. Let $f: \Omega \rightarrow \mathbb{R}$, with $f \in C^{0,1}(\Omega)$, an L-Lipschitz function and let $\delta_j =  \frac{f(p_j)}{L},\; \forall j \in I = \{1,2,\dots,k\},$ be the radius of certified positivity given by Proposition \ref{prop:rad_mono}. Then if 
    \begin{equation}\label{eq:ap_suff_condition}
        \max_{x \in R_j} d(x,p_j) < \delta_j \; \forall j \in I
    \end{equation}
    the function $f$ is positive in $\Omega$.
\end{theorem}
\begin{proof}
    Let $x$ be a point in the compact domain $\Omega$ and suppose that $f$ verifies Eq. \eqref{eq:ap_suff_condition}. Therefore, by Lemma \ref{lem:covering}, 
    \(
        \Omega\subseteq\bigcup_{i \in I} B(p_i, \delta_i).
    \)
    Consequently, by definition of the Voronoi diagram, there exists a Voronoi cell $R_j$ such that $x \in R_j \subseteq B(p_j,\delta_j)$. Therefore, by Proposition \ref{prop:rad_mono}, $f(x) > 0$. Consequently, f is positive in the whole compact domain $\Omega$.
\end{proof}

\section{Proof of Theorem \ref{teo:convergence}}\label{sec:ap_bound_iterations}
Recall that the \hyperref[alg:alg1]{LipVor Algorithm} provides a mathematical certification to check if a function $f$ is positive in a compact domain $\Omega$. As mentioned in section \ref{subsec:lipvor_algo}, one aspect worth mentioning is the finiteness of the LipVor Algorithm. Therefore,  this appendix provides a proof of Theorem \ref{teo:convergence}, which states that the LipVor Algorithm reaches a conclusion in a finite number of steps and gives an upper bound of the maximum number of iterations needed.

Let us start by recalling the definition of $\varepsilon$-positivity that will be later used in Theorem \ref{teo:ap_convergence}.
Let $f: \Omega \subseteq\mathbb{R}^n \rightarrow \mathbb{R}$ be a function in a domain $\Omega$, then the function $f$ is said to be $\varepsilon$-positive, with $\varepsilon >0 $, if
\(
    f(x) \geq \varepsilon, \forall \; x \in \Omega.
\)
Although $\varepsilon$-positivity might seem like a more restrictive condition than positivity, the following lemma will prove that every positive function in a compact domain $\Omega$ is indeed $\varepsilon$-positive for some $\varepsilon$. Moreover, it is trivial that an $\varepsilon$-positive function is positive. Therefore, for continuous functions in compact domains, $\varepsilon$-positivity is equivalent to positivity.

\begin{lemma}\label{lem:eps_strict_pos}
    Let $f: \Omega \subseteq\mathbb{R}^n \rightarrow \mathbb{R}$ be a positive continuous function in a compact domain $\Omega$. Then, there exists an $\varepsilon_{f,\Omega} >0$ such that
    \(
        f(x) \geq \varepsilon_{f,\Omega}, \forall \; x \in \Omega.
    \)
\end{lemma}
\begin{proof}
    By continuity of $f$, the image of the compact domain $\Omega$ is again a compact space in $\mathbb{R}$. Therefore,  by the Heine-Borel Theorem \cite{Rudin1976PrinciplesMathematics}, $f(\Omega)$ is closed and bounded. Hence, f attains its minimum
    \(
        \varepsilon_{f,\Omega} = \min\left( f(\Omega)\right),
    \) in $f(\Omega)$. Besides, as f is positive, then  $\varepsilon_{f,\Omega} >0$. Thus, by definition of the minimum of a set
    \(
        f(x) \geq \varepsilon_{f,\Omega}, \forall \; x \in \Omega.
    \)
\end{proof}
Consequently, considering $\varepsilon$-positive functions is just a formal requirement of Theorem \ref{teo:ap_convergence} that can be easily translated to positivity of a function $f$ in $\Omega$ by just decreasing $\varepsilon$. 
  
\begin{theorem}\label{teo:ap_convergence}
    Let $f:\Omega \subseteq\mathbb{R}^n \rightarrow \mathbb{R} $ be an L-Lipschitz function, $f \in C^{0,1}(\Omega)$ where $\Omega$ is a compact domain on which the positivity wants to be certified.
    Then the maximum number of iterations $N$ needed by the \hyperref[alg:alg1]{LipVor Algorithm} to fill $\Omega$, and therefore to certify positivity, or to find a non-$\varepsilon$-positive counter-example is upper bounded by 
    \begin{equation} \label{eq:ap_max_iter_bound}
    N \leq \frac{\Vol(\bar{\Omega})}{\Vol \left(B\left(\frac{\varepsilon}{2L}\right)\right)}, 
    \end{equation}
    where $B\left(\frac{\varepsilon}{2L}\right)$ is the ball of radius $\frac{\varepsilon}{2L}$ and $\bar{\Omega}$ is the Minkowski sum  of the domain $\Omega$ and $B\left(\frac{\varepsilon}{2L}\right)$ .
\end{theorem}
\begin{proof}
Let us suppose that, after executing the LipVor Algorithm for $N-1$ iterations, the sufficient condition stated in Theorem \ref{teo:ap_glob_condition} is not fulfilled and therefore $f$ is not yet certified positive in $\Omega$. Therefore, there are $N-1$ points $\mathcal{P}=\{p_1,\dots,p_{N-1}\}\subseteq \Omega$ selected by LipVor such that $f(p_i) \geq \varepsilon, \forall \, i \in I= \{1,2,\dots,N-1\}$ (otherwise there would be already a counter-example of $\varepsilon$-positivity) but not verifying Eq. \eqref{eq:ap_suff_condition}. 

By Proposition \ref{prop:rad_mono}, for every $p_i \in \mathcal{P}$ there exists a  radius $\delta_i$ of positivity verifying
\(
    \delta_i = \frac{f (p_i)}{L} \geq \frac{\varepsilon}{L},
\)
where L is the Lipschitz constant of $f$. Moreover, considering $(R_i)_{i\in I}$ the Voronoi cells generated by $\mathcal{P}$ in $\Omega$, let 
\(
    \tilde{\mathcal{P}} = \{ p_i \in \mathcal{P} \mid  \max_{x \in R_i} d(x,p_i) \geq \delta_i\}
\)
be the subset of points of $\mathcal{P}$ such that the radius of certified positivity is not enough to cover its corresponding Voronoi cell.
Then, by hypothesis, $\tilde{\mathcal{P}} \not = \emptyset$ as Eq. \eqref{eq:ap_suff_condition} is not fulfilled. Therefore, for every $p_i \in \tilde{\mathcal{P}}$, it is verified that 
\(
     \max_{x \in R_i} d(x,p_i) \geq \frac{f(p_i)}{L} \geq \frac{\varepsilon}{L}.
\)
Consequently, for any point $p_i \in \tilde{\mathcal{P}}$ the distance to its furthest vertex $v_i$ verifies that  $d(p_i,v_i)\geq \frac{\varepsilon}{L}$. Therefore $d(p_j,v_i) \geq \frac{\varepsilon}{L}$, $\forall \, j  \in  I$ by definition of the Voronoi diagram. Moreover, as every $p_i \in \mathcal{P}$ has been selected using the LipVor Algorithm, in particular $d(p_i,p_j) \geq \frac{\varepsilon}{L}, \forall \; i\not = j$ for the same reason explained for the case of $v_i$.

Taking this into account, there are N disjoint open balls of radius $\delta = \frac{\varepsilon}{2L}$ centered at the set of points $\mathcal{P} \cup \{v_i\}$ such that 
\(
    B(p_i,\delta) \cap B(p_j,\delta) = \emptyset, \forall \, i,j \in I, i\not=j,
\)
and 
\(
    \bigcup_{i \in I} B(p_i, \delta) \subset \bar{\Omega},
\)
where $\bar{\Omega}$ is the domain $\Omega$ extended by $\frac{\varepsilon}{2L}$ given by 
\(
\bar{\Omega}= \Omega + B \left(x, \frac{\varepsilon}{2L}\right) = \bigcup_{x\in \Omega} B \left(x, \frac{\varepsilon}{2L}\right) .
\)
with $+$ representing the Minkowski sum.
Therefore,
\(
     \Vol \left(\bigcup_{i \in I} B(x_i, \delta) \right) = N \cdot \Vol \left(B\left(x_i, \frac{\varepsilon}{2L}\right)\right) \leq \Vol(\bar{\Omega}),
\)
which implies that
\(
    N \leq \frac{\Vol(\bar{\Omega})}{\Vol \left(B\left( \frac{\varepsilon}{2L}\right)\right)}.
\)
\end{proof}
Lastly, if $\Omega \subseteq \mathbb{R}^n$ is a compact domain, then by the Heine-Borel Theorem \cite{Rudin1976PrinciplesMathematics}, $\Omega$ is closed and bounded. Therefore, up to translation, there exists an n-dimensional hyperrectangle $H^n = [0,a_1] \times [0,a_2] \times \dots \times [0,a_n]$ with $a_i >0, \forall \; 1 \leq i \leq n$ such that $\Omega \subseteq H^n$. Consequently, 
\(
    \Vol (\bar{\Omega}) \leq \Vol (H^n) = \prod_{i=1}^{n} \left(a_i + \frac{\varepsilon}{L}\right)^n.  
\)
Therefore, as the volume of an n-dimensional ball is given by
\(
    \Vol(B(R)) = \frac{\pi^{\frac{n}{2}}\cdot R^n}{\Gamma \left(\frac{n}{2}+1\right)},
\)
where $\Gamma$ is the gamma function, then
\begin{align*}
    N &\leq \frac{\Vol(\bar{\Omega})}{\Vol \left(B\left( \frac{\varepsilon}{2L}\right)\right)}
    \leq
    \frac{\left(a_i + \frac{\varepsilon}{L}\right)^n}{\left( \frac{\pi^{\frac{n}{2}}\cdot \left( \frac{\varepsilon}{2L}\right)^n}{\Gamma \left(\frac{n}{2}+1\right)} \right)} \\
    &\leq
    \left(\frac{2\cdot \left(a\cdot L + \varepsilon\right)}{\pi^{\frac{1}{2}}\cdot\varepsilon}\right)^n \Gamma \left(\frac{n}{2}+1\right),
\end{align*}
considering $a = \max_{1\leq i \leq n} \{a_i \}$. Consequently, the LipVor algorithm clearly concludes in a finite number of steps.

\section{Upper Bound of the Lipschitz Constant of an ANN's Partial Derivative}\label{lips_appendix}
Given the tensor formulation of an ANN proposed in \cite{Gonzalo2023ExplainableNetworks}, the Jacobian matrix and Hessian tensor can be described in terms of the weight tensors and the Jacobians and Hessians of its layers. We will use this description to compute the necessary upper bounds for the Hessian. Therefore, the first step will be to describe the aforementioned formulation. 

Recall that an ANN $g:\mathbb{R}^n \rightarrow \mathbb{R}$ can be described as a composition of linear transformations with activation functions such that
\begin{align}
    o^0 &= \mathbf{x}, \nonumber \\
    z^l &= o^{l-1} \cdot W^l + b^l \quad \text{for } l = 1, 2, \ldots, K,  \label{eq:ap_z_l}\\\
o^l &= \varphi^l \left( z^l \right), \quad \text{for } l = 1, 2, \ldots, K,  \nonumber\\
    y &= g(\mathbf{x}; \mathbf{W}, \mathbf{b}) =  \varphi^K \left( \mathbf o^{K-1} \cdot W^K + b^K \right) \nonumber,
\end{align}
where $\mathbf{x}$ represents the input data, $W^l$ represent the weights, $b^l$ the bias and $\varphi^l$ the activation function of the $l^{th}$ layer. Therefore, under the aforementioned notation, the 0-layer (input layer) can be considered as a linear transformation where $W^0=I_{n}$, $b_0=\mathbf{0}_n$ and $\varphi^0 (\mathbf{x}) = \mathbf{x}$. Moreover, following Eq. \eqref{eq:ap_z_l}, the weight matrix of the $l^{th}$ layer is given by 
\(
    W^l = \left(\frac{\partial z^l_k}{\partial o^{l-1}_j}\right)_{j,k} \in \mathbb{R}^{n^{l-1}}\times\mathbb{R}^{n^l}, 0\leq j \leq n^{l-1},  \leq k \leq{n^l}.
\)

According to the methodology proposed in \cite{Gonzalo2023ExplainableNetworks}, the Jacobian matrix and the Hessian tensor of the $l^{th}$ layer w.r.t. the inputs in the $p^{th}$ layer will be denoted by $J_p^l$ and $H_p^l$ respectively. Consequently, given an ANN with K layers, the Jacobian matrix $J_0^l \in \mathbb{R}^{n^0}\times\mathbb{R}^{n^l}$ of the $l^{th}$ layer with $n^l$ neurons, with $0<l\leq K$,  w.r.t. the $n^0=n$ inputs is given by
\begin{equation}\label{eq:ap_jac_tensor}
    J_0^l = J_0^{l-1} \cdot W^l \cdot J_l^l,
\end{equation}
where $J_l^l = \left(\frac{\partial \varphi^l_j(z^l)}{\partial z^l_i}\right)_{i,j}$, with $1\leq i,j \leq n^l$, is the Jacobian matrix of the output of the $l^{th}$ layer w.r.t. $z^l$. 

On the other hand, the Hessian of the $l^{th}$ layer w.r.t. the inputs can be equally stated in its tensor form as
\begin{multline}\label{eq:ap_hess_tensor}
    H_0^l = \left(J_0^{l-1}\cdot W^l\right) \otimes_i \left(J_0^{l-1}\cdot W^l\right)  \otimes_j H_l^l \\+ H_0^{l-1}\otimes_k (W^k \cdot J_l^l),
\end{multline}
where $H_l^l =\left(\frac{\partial \varphi^l(z^l_k)}{\partial z^l_j \partial z^l_i}\right)_{i,j,k} \in \mathbb{R}^{n^l}\times\mathbb{R}^{n^l}\times\mathbb{R}^{n^l}$ is the 3D Hessian tensor of the output of the $l^{th}$ layer w.r.t. the input of that layer. Moreover, $\otimes_a$ is the tensor multiplication of an $n$-dimensional tensor by a matrix along each of the layers of the $a$ axis 
\cite[page 50]{Gonzalo2023ExplainableNetworks}.
For instance, following the Einstein notation, considering a 3D tensor $T = \left(T_{ijk}\right)$ and a matrix $M =\left(M_{lm}\right) $, then $\otimes_i$ is the multiplication along the layers of the $i$-axis of $T$, given by 
\(
    (T\otimes_i M)_{ijk}=\sum_t T_{ijt}M_{tk}, \qquad  (M\otimes_i T)_{ijk}=\sum_t M_{jt}T_{itk}.
\)
Under the aforementioned notation, it can be stated the following theorem that establishes an upper bound for the Hessian of an ANN and therefore an upper estimate $\hat{L}$ of the Lipschitz constant $L$.
\begin{theorem}\label{teo:ap_lip_acot}
        Let $g:\Omega \subseteq\mathbb{R}^n \rightarrow \mathbb{R}$ be an ANN with K-layers, then the Lipschitz constant $ L_{x_r}^K$ of the $r^{th}$ partial derivative $g_{x_r}$ of an ANN verifies that
    \begin{multline*}
        L_{x_r}^K \leq \hat{L}_{x_r}^K := \max|a_k^K|\cdot\lVert W_{1j}^1\rVert \cdot\lVert W^1\rVert\cdot\lVert W^2\rVert ^2\cdot\, \dots \,\cdot\lVert W^K\rVert ^2\\ +\hat{L}_{x_r}^{K-1} 
        \cdot\lVert W^K \rVert,
    \end{multline*}
    where $W^k$ is the weight matrix of the $k^{th}$ layer, $W_{1j}^1$ is the first row of the weight matrix $W^1$, $\hat{L}_{x_r}^{K-1}$ is the upper bound of Lipschitz constant of the $r^{th}$ partial derivative of the $K-1$ layer of the ANN, $a_k^K = \|H_K^K\|_{\infty}:= \max_{1\leq i\leq n}\left(\sum_{j=1}^n \left(H_K^K\right)_{i,j}\right)$ and rest of the matrix norms are $L^2$ norms.
\end{theorem}
\begin{proof}
    Let us prove the theorem by induction over the number of layers of the NN.
    
    \textit{Base Case} $\left(n=1\right)$

    According to Eq. \eqref{eq:ap_hess_tensor}, the Hessian tensor $H_0^1$ can be written as 
    \(
        H_0^1 = \left(J_0^{0}\cdot W^1\right) \otimes_i \left(J_0^{0}\cdot W^1\right)  \otimes_j H_1^1 + H_0^{0}\otimes_k (W^1 \cdot J_1^1),
    \)
    with $H_0^1 \in \mathbb{R}^{n^0}\times\mathbb{R}^{n^0}\times \mathbb{R}^{n^1}$.
    
    However, finding an upper bound for the Lipschitz constant of the $ r^{th} $ partial derivative $ L_{x_r} $ does not require establishing a bound for the norm of $ H_0^1 $, but rather for
    \[\tilde{H}_0^1 = e_r \otimes_i H_0^1 = \left(\frac{\partial o^1_k}{\partial x_j \partial x_r}\right)_{j,k}, 1\leq j,k \leq n^l,\]
    where
    $e_r=(0,\ldots,0,\underset{\stackrel{\frown}{r}}{1},0,\ldots,0)$
    is the $r^{th}$ vector of the canonical base 
    \cite[page 50]{Gonzalo2023ExplainableNetworks}.
    Therefore, Eq. \eqref{eq:ap_hess_tensor} reduces to
    \begin{align*}
        \tilde{H}_0^1 &= \left(\left(e_r \otimes_i J_0^{0}\right)\cdot W^1\right) \otimes_i \left(J_0^{0}\cdot W^1\right)  \otimes_j H_1^1\\
        &\quad+ \left(e_r \otimes_i H_0^{0}\right)\otimes_k (W^1 \cdot J_1^1)\\
        &= \left(\tilde{J}_0^{0}\cdot W^1\right) \otimes_i \left(J_0^{0}\cdot W^1\right)  \otimes_j H_1^1 + \cancelto{0}{\tilde{H}_0^{0}}\otimes_k (W^1 \cdot J_1^1)\\
        &=\left(e_r\cdot W^1\right) \otimes_i \left(I_{n^l}\cdot W^1\right)  \otimes_j H_1^1\\
        &= \left(W^1_{1j} \otimes_i H_1^1\right)\cdot \left(W^1\right)^t.
    \end{align*}
    Therefore, 
    $ \|\tilde{H}_0^1\| \leq \underset{\text{I}}{\underbrace{\norm{W^1_{1j} \otimes_i H_1^1}}} \cdot \underset{\text{II}}{\underbrace{\norm{W^1}}},$
    where $W_{1j}^1 \in \mathbb{R}^{n^1}$ is the first row of the matrix of weights $W^1 \in \mathbb{R}^{n^0}\times \mathbb{R}^{n^1}$. On the other hand, $H_1^1 = \mathcal{D}^3(a_1^1,\dots,.a_{n^1}^1)$ is a diagonal 3D tensor, as the output of the $k^{th}$ neuron of the first layer $o^1_k = \varphi^1(z_k^1)$ only depends on $k^{th}$ input $z_k^1$. Hence
    
    
    \begin{align}\label{eq:ten_bound}
        &W^1_{1j} \otimes_i H_1^1 = (w_{11}^1,\dots,w_{1n^1}^{1})\otimes_i\mathcal{D}^3(a_1^1,\dots,.a_{n^1}^1) \nonumber \\
        &= \mathcal{D}^2(w_{11}^1\cdot a_1^1,\cdots,  w_{1n^1}^1\cdot a_{n^1}^1).
    \end{align}

    Hence, 
    $\text{I}=\norm{W^1_{1j} \otimes_i H_1^1} \leq \max_{k \in \{1,\dots, n^1\}}|a_k^1|\cdot \norm{W_{1j}^1}$
    and 
    \begin{align*}
        \|\tilde{H}_0^1\| &\leq \hat{L}^1_{x_r}:= \underset{\text{I}}{\underbrace{\norm{W^1_{1j} \otimes_i H_1^1}}} \cdot \underset{\text{II}}{\underbrace{\norm{W^1}}} \\
        & = \max_{k \in \{1,\dots, n^1\}}|a_k^1|\cdot \norm{W_{1j}^1} \cdot \norm{W^1}.
    \end{align*}
 
    \textit{Induction Step}. Suppose the case is true for $l-1$ layers and let us prove it for the $l^{th}$ layer.
    
    If the ANN has $l$ layers, then according to Eq. \eqref{eq:ap_hess_tensor},
    \begin{equation}\label{eq:ap_h_0_l}
        H_0^l = \left(J_0^{l-1}\cdot W^l\right) \otimes_i \left(J_0^{l-1}\cdot W^l\right)  \otimes_j H_l^l + H_0^{l-1}\otimes_k (W^l \cdot J_l^l).
    \end{equation}
    Therefore, to find $\hat{L}^l_{x_r}$ is it necessary to find an upper bound of the norm of  $\tilde{H}_0^l = e_r \otimes_i H_0^l$.
    Consequently, using Eq. \eqref{eq:ap_h_0_l}, 
    \begin{align*}
        &\tilde{H}_0^l = \left(\left(e_r \otimes_i J_0^{l-1}\right)\cdot W^l\right) \otimes_i \left(J_0^{l-1}\cdot W^l\right)  \otimes_j H_l^l \nonumber \\
        &\quad+ \left(e_r \otimes_i H_0^{l-1}\right)\otimes_k (W^l \cdot J_l^l) \nonumber \\
        &= \left(\tilde{J}_0^{l-1}\cdot W^l\right) \otimes_i \left(J_0^{l-1}\cdot W^l\right)  \otimes_j H_l^l + {\tilde{H}_0^{l-1}}\otimes_k (W^l \cdot J_l^l) \nonumber \\
        &= \underset{\text{I}}{\underbrace{\left(\left(\tilde{J}_0^{l-1}\cdot W^l\right) \otimes_i H_l^l \right) \cdot  \left(J_0^{l-1}\cdot W^l\right)^t}}  + \underset{\text{II}}{\underbrace{\left(\tilde{H}_0^{l-1} \cdot (W^l \cdot J_l^l)\right)}} .
    \end{align*}
    Thus, following the triangle inequality $\norm{\tilde{H}_0^l}= \norm{\text{I}+\text{II}}\leq \norm{\text{I}} + \norm{\text{II}}$. Hence, after finding an upper bound for I and II, it can be found an upper bound for $\norm{\tilde{H}_0^l}$.

    Supposed true the inductive hypothesis for $l-1$, then
    $\norm{\tilde{H}_0^{l-1}} \leq \tilde{L}_{x_r}^{l-1},$
    and therefore, as $\tilde{H}_0^{l-1} \in \mathbb{R}^{n^0} \times  \mathbb{R}^{n^{l-1}}$, then  
    \begin{multline*}
        \norm{\text{II}} = \norm{\tilde{H}_0^{l-1} \cdot (W^l \cdot J_l^l)} \leq \norm{\tilde{H}_0^{l-1}}\cdot \norm{W^l} \cdot \norm{J_l^l} \\
        \leq \tilde{L}_{x_r}^{l-1} \cdot \norm{W^l}.
    \end{multline*}
    Therefore, we have already found an upper bound for $\norm{\text{II}}$. On the other hand, $\norm{\text{I}}$ can be upper bounded by considering that 
    \(
        \norm{\text{I}} \leq \norm{\left(\left(\tilde{J}_0^{l-1}\cdot W^l\right) \otimes_i H_l^l \right)} \norm{ J_0^{l-1}}\cdot \norm{W^l}.
    \)

    First of all, according to Eq. \eqref{eq:ap_jac_tensor}, the Jacobian $J_0^l$ of the $l^{th}$ layer w.r.t. the input layer can be described as 
    \(
    J_0^l = J^0_0\cdot W^1 \cdot J^1_1 \cdot \, \dots \,\cdot W^{l}\cdot J_l^l,
    \)
    and therefore $ \norm{ J_0^{l}} \leq \prod_{1\leq i \leq l} \norm{W^l}$.

    Moreover, $\tilde{J}_0^{l-1}\cdot W^l \in \mathbb{R}^{n^l}$. Consequently, as $H_l^l$ is a 3D diagonal tensor such that $H_l^l = \mathcal{D}^3(a_1^1,\ldots,a_{n^2}^2)$, then by the same principle followed in  Eq. \eqref{eq:ten_bound}, 
    \(
        \norm{\left(\tilde{J}_0^{l-1}\cdot W^l\right) \otimes_i H_l^l} \leq \max_{k \in \{1,\dots, n^l\}}|a_k^l|\cdot \norm{\tilde{J}_0^{l-1}\cdot W^l},
    \)
    and then by Eq. \eqref{eq:ap_jac_tensor},
    \begin{align*}
        &\tilde{J}_0^{l-1}\cdot W^l = e_r\cdot J_0^0 \cdot W^1 \cdot J^1_1 \cdot \, \dots \, \cdot W^l\cdot J_l^l \implies \\
        &\norm{\tilde{J}_0^{l-1}\cdot W^l} \leq \norm{W^1_{1j}} \cdot \norm{W^2} \cdot \, \dots \, \cdot \norm{W^l}.
    \end{align*}
    Consequently,
    \begin{multline*}
        \norm{\text{I}} \leq \norm{\left(\left(\tilde{J}_0^{l-1}\cdot W^l\right) \otimes_i H_l^l \right)} \norm{ J_0^{l-1}}\cdot \norm{W^l} \\
        \leq \max_{k \in \{1,\dots, n^l\}}|a_k^l| \cdot \norm{W^1_{1j}} \cdot \norm{W^2} \cdot \, \dots \, \cdot \norm{W^l} (\prod_{1\leq i \leq l} \norm{W^l}) = \\
        \max_{k \in \{1,\dots, n^l\}}|a_k^l| \cdot \norm{W^1_{1j}} \cdot \norm{W^1}\cdot \norm{W^2}^2 \cdot \, \dots \, \cdot \norm{W^l}^2.
    \end{multline*}

    Consequently, as $\norm{\tilde{H}_0^2} \leq \norm{\text{I}}+\norm{\text{II}}$, then 
    
    \begin{multline*}
        L_{x_r}^K \leq \hat{L}_{x_r}^K := \max|a_k^K|\cdot\lVert W_{1j}^1\rVert \cdot\lVert W^1\rVert\cdot\lVert W^2\rVert ^2\cdot\, \dots \,\cdot\lVert W^K\rVert ^2\\ +\hat{L}_{x_r}^{K-1} 
        \cdot\lVert W^K \rVert.
    \end{multline*}

\end{proof}

\printbibliography

\end{document}